\documentclass[10pt,journal,compsoc]{IEEEtran}
\usepackage{booktabs}
\usepackage{multirow}
\usepackage{makecell}
\usepackage{amsthm}
\usepackage{diagbox}
\usepackage{xcolor}
\usepackage{subfigure}
\usepackage{amsmath,amssymb,amsfonts}
\usepackage{algorithm}
\usepackage{algorithmicx}
\usepackage{algpseudocode}
\usepackage{ragged2e}
\usepackage{cuted}
\usepackage{graphicx}
\usepackage{epstopdf}

\usepackage{booktabs}
\usepackage{array}
\usepackage{lscape} 
\usepackage{float}

\usepackage[colorlinks, linkcolor=red,anchorcolor=blue,citecolor=green]{hyperref}

\newtheorem{myDef}{Definition}
\newtheorem{myTheo}{Theorem}

\ifCLASSOPTIONcompsoc
\usepackage[nocompress]{cite}
\else
  \usepackage{cite}
\fi

\ifCLASSINFOpdf

\else

   \usepackage[dvips]{graphicx}

\fi

\begin{document}
%
\title{Privacy-Aware Joint DNN Model Deployment and Partitioning Optimization for Collaborative Edge Inference Services}
%
%
%

\author{Zhipeng Cheng, Xiaoyu Xia,~\IEEEmembership{Senior Member,~IEEE,} Hong Wang, Minghui Liwang,~\IEEEmembership{Senior Member,~IEEE,} Ning Chen, Xuwei Fan, and Xianbin Wang,~\IEEEmembership{Fellow,~IEEE}

\thanks{Zhipeng Cheng (chengzp\_x@163.com) and Hong Wang (wh\_5233@163.com) are with School of Future Science and Engineering, Soochow University, Suzhou 215006, China. Xiaoyu Xia (xiaoyu.xia@rmit.edu.au) is with the School of Computing Technologies, RMIT University, Melbourne, Victoria, Australia. Minghui Liwang (minghuiliwang@tongji.edu.cn) is with the Department of Control Science and Engineering,
 The National Key Laboratory of Autonomous Intelligent Unmanned Systems, Tongji University, Shanghai 201804, China, and also
with the Frontiers Science Center for Intelligent Autonomous Systems, Ministry of Education, Tongji University, Shanghai 201804, China. Ning Chen (chenning@upc.edu.cn) is with Department of automation, China University of Petroleum (East China), Qingdao, China. Xuwei Fan (xwfan@fafu.edu.cn) is with College of Computer and Information Sciences, Fujian Agriculture and Forestry University, Fuzhou 350002, China.   Xianbin Wang (xianbin.wang@uwo.ca) is with Department of Electrical and Computer Engineering, Western University, Ontario, Canada. The Corresponding author is Minghui Liwang.}
}

%
%

\markboth{}%
{Shell \MakeLowercase{\textit{et al.}}:}
%



\IEEEtitleabstractindextext{%
\begin{abstract}
\justifying
Edge inference (EI) has emerged as a promising paradigm to address the growing limitations of cloud-based Deep Neural Network (DNN) inference services, such as high response latency, limited scalability, and severe data privacy exposure. However, deploying DNN models on resource-constrained edge devices introduces additional challenges, including limited computation/storage resources, dynamic service demands, and heightened privacy risks. To tackle these issues, this paper presents a novel privacy-aware optimization framework that jointly addresses DNN model deployment, user-server association, and model partitioning, with the goal of minimizing long-term average inference delay under resource and privacy constraints. The problem is formulated as a complex, NP-hard stochastic optimization. To efficiently handle system dynamics and computational complexity, we employ a Lyapunov-based approach to transform the long-term objective into tractable per-slot decisions. Furthermore, we introduce a coalition formation game to enable adaptive user-server association and design a greedy algorithm for model deployment within each coalition. Extensive simulations demonstrate that the proposed algorithm significantly reduces inference delay and consistently satisfies privacy constraints, outperforming state-of-the-art baselines across diverse scenarios.
\end{abstract}

}

\maketitle

\IEEEdisplaynontitleabstractindextext

%
\IEEEpeerreviewmaketitle

\IEEEraisesectionheading{\section{Introduction}\label{sec:introduction}}

\IEEEPARstart{N}{owadays}, numerous applications based on Deep Neural Networks (DNNs) rely on remote cloud servers to perform complex inference services~\cite{c1}. However, this approach introduces several challenges, including delayed responses due to long-distance network transmissions, scalability issues arising from high bandwidth consumption, and privacy concerns related to the forwarding and storage of user data over the network~\cite{c2}. Furthermore, the growing disparity between the limited resources on mobile devices (MDs) and the computational demands of inference services presents another significant hurdle~\cite{c4}. As DNN models increase in size and complexity, both storage and computational requirements have escalated, making local processing increasingly difficult~\cite{c6}. For example, the VGG19 model, with 143 million parameters and a storage requirement of about 600 MB, poses substantial obstacles for execution on resource-constrained MDs~\cite{c7}.

To address these challenges, edge inference (EI) has emerged as a promising solution, enabling low-latency decision-making by deploying DNN inference services at the edge~\cite{c8,c9,c10,c11}. This approach involves offloading user requests to edge servers, which collaborate with MDs to complete inference tasks. By partitioning DNN models and distributing parts of the computation across both edge servers and MDs, collaborative EI has the potential to reduce inference delay and energy consumption, alleviate privacy concerns, and capitalize on the computational power of edge servers~\cite{c12,c13}.

Despite its promising potential, collaborative EI faces two major challenges:
\begin{itemize}
  \item \textbf{Resource Constraints}: One of the primary obstacles is the limited storage and computational resources available on MDs and edge servers, especially in the context of dynamic and diverse user requests~\cite{c14}. Optimizing model deployment and user-server association among MDs and edge servers is crucial to overcoming these resource limitations. Model deployment determines how DNN models are stored and accessed on edge servers to reduce inference delay and improve service quality~\cite{c15}. An effective deployment strategy optimizes model caching, reduces download and transmission times, and dynamically allocates resources based on request probability and computational requirements, thereby enhancing overall system responsiveness. Similarly, user association strategies aim to assign user requests to the most suitable edge servers, balancing system load and maximizing resource utilization while minimizing communication overhead. However, the limited resources of edge servers and the diversity of user requests make this optimization complex, requiring careful coordination of both model deployment and user association to fully leverage available resources and provide efficient, delay-minimized inference services.
  \item \textbf{Privacy-Communication-Computation Trilemma}: Collaborative inference introduces a fundamental trilemma among privacy, communication, and computation. Specifically, model partitioning directly impacts the computational burden on MDs and edge servers, as well as the communication overhead between them~\cite{c16,c17}. While partitioning helps reduce the need for MDs to transmit raw data—since only processed feature maps from early layers are sent—it still exposes the system to privacy risks. These feature maps are susceptible to model inversion attacks, which may lead to privacy breaches and the compromise of sensitive data~\cite{c18,c19}. To highlight this, our experiments (see Figs. \ref{fig1} and \ref{fig2}) measure privacy leakage during inference using data similarity metrics such as the Structural Similarity Index Measure (SSIM), Peak Signal-to-Noise Ratio (PSNR), and Learned Perceptual Image Patch Similarity (LPIPS)~\cite{c20}. These metrics illustrate the extent of privacy leakage under potential inversion attacks. This creates a fundamental trade-off: deeper local computation enhances privacy (lower SSIM), while shallower partitions reduce the local computational burden but increase both communication overhead and privacy risks. As a result, an effective model partition strategy must not only optimize computational and communication efficiency, but also carefully manage privacy concerns.
\end{itemize}

\begin{figure}[t]
    \centering
    \subfigure[] {\includegraphics[width=1.1in,angle=0]{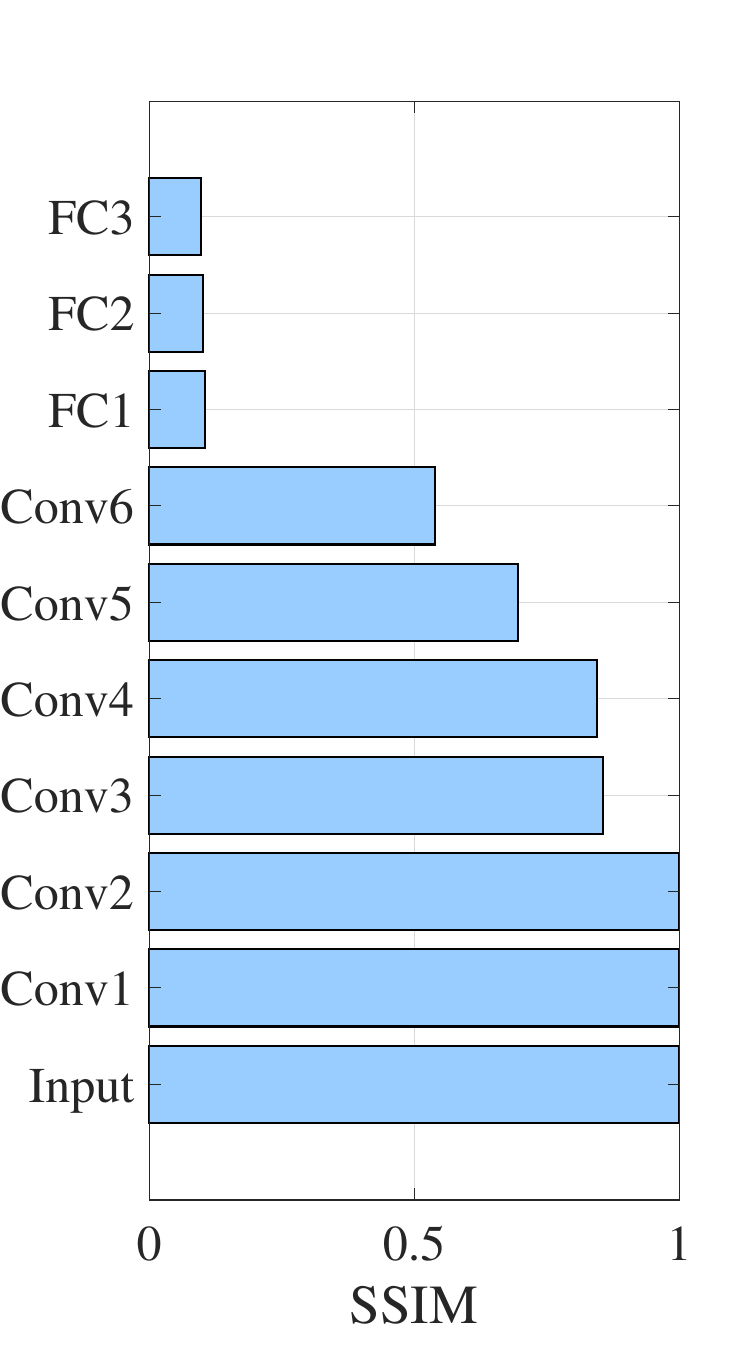}}
    \subfigure[] {\includegraphics[width=1.1in,angle=0]{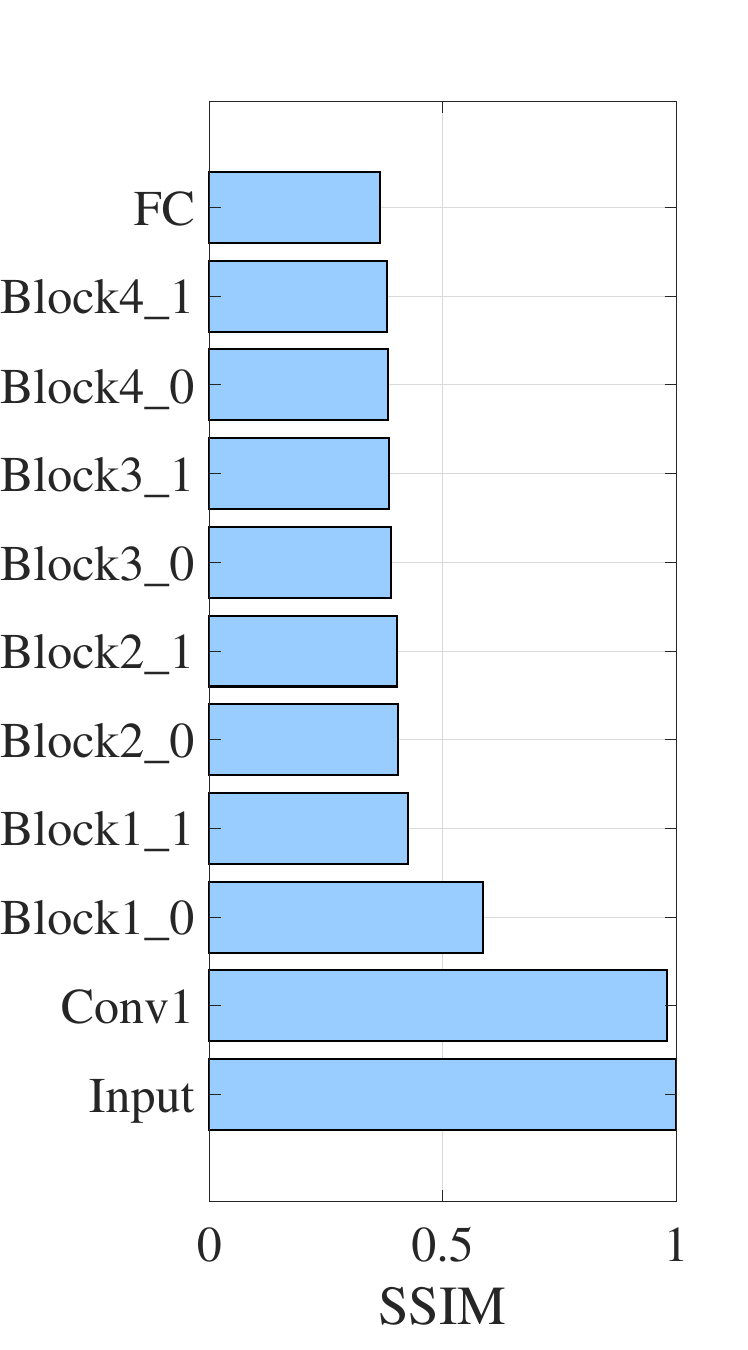}}
    \subfigure[] {\includegraphics[width=1.1in,angle=0]{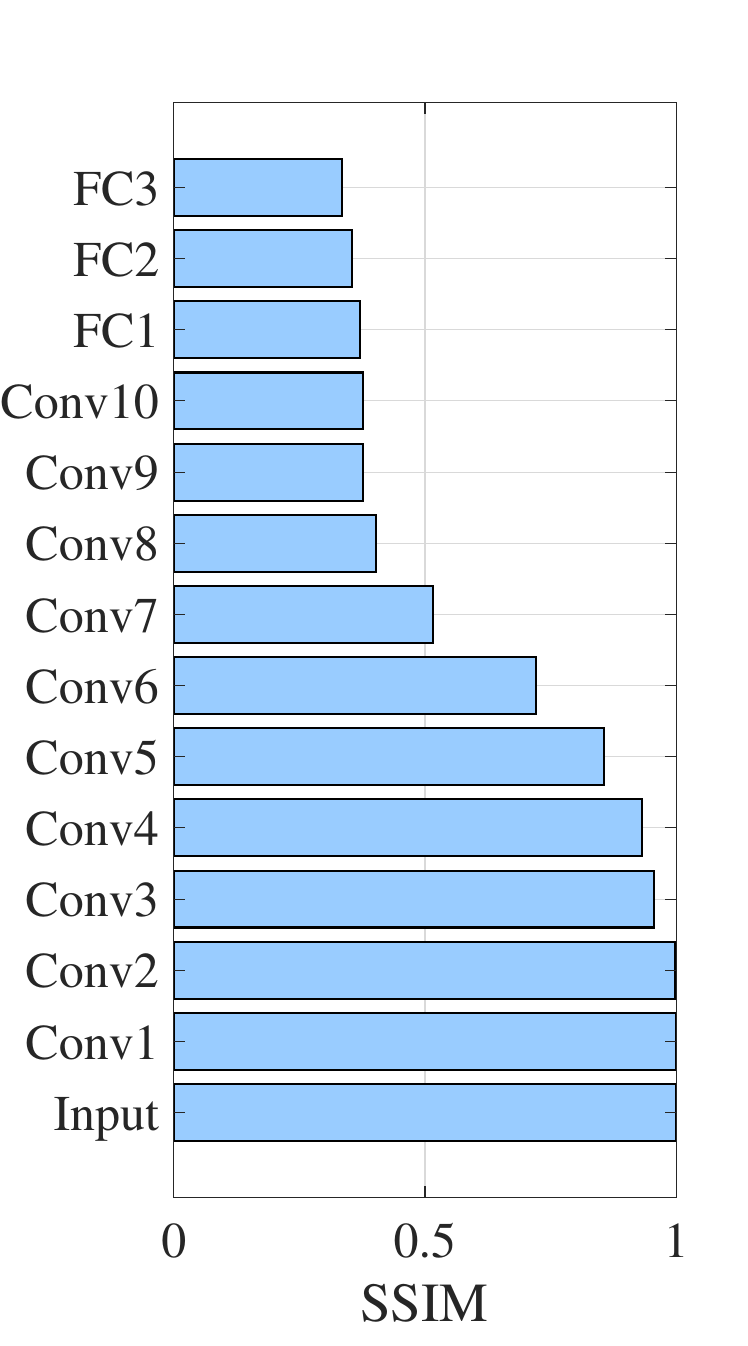}}
   \caption{Evaluation of the SSIM of images reproduced from the intermediate feature maps leaked from different layers of different models and data sets: (a) LeNet12 on CIFAR-10; (b) ResNet18 on CIFAR-100; (c) VGG13 on Caltech-101.}
    \label{fig1}
\end{figure}
\begin{figure}[t]
	\centering
	\includegraphics[width=3.35in]{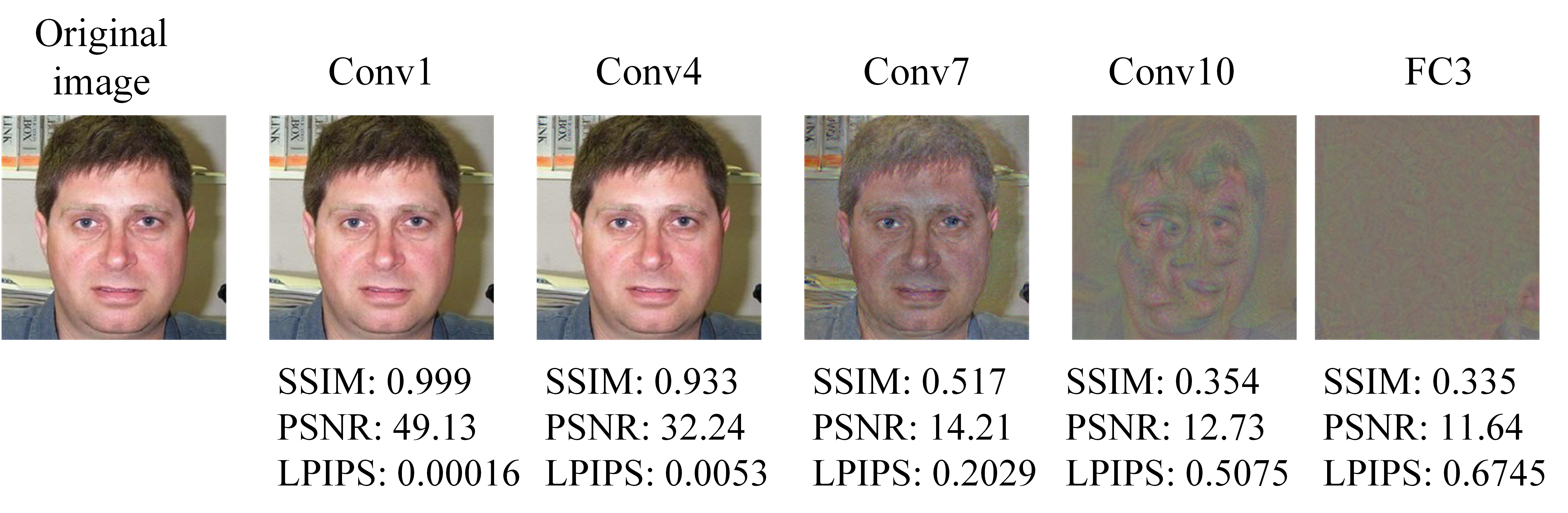}
	\caption{Illustration of reproduced image and different metrics after different layers of VGG13.}
	\label{fig2}
\end{figure}

Building on these insights, this paper aims to strike a balance between minimizing inference delay and safeguarding user privacy, thereby enhancing the overall performance of EI. Specifically, given an edge network composed of multiple inference services based on various DNN models, edge servers, and MDs, we propose a joint optimization approach that integrates model deployment, user association, and model partitioning, aiming to reduce long-term average inference delay while adhering to resource limitations and privacy constraints. Notably, while previous research has addressed model deployment, model partitioning, or user association individually, to the best of our knowledge, no existing work has jointly optimized all three elements to minimize inference delay under privacy and resource constraints.

The key contributions of this paper are as follows:
\begin{itemize}
  \item We formulate a novel joint optimization problem for multi-edge environments, aiming to minimize the long-term average inference delay while respecting the resource limitations of edge servers and privacy constraints. We further prove that this problem is NP-hard, highlighting its complexity and the necessity for suboptimal solutions.
  \item We propose a Lyapunov-based approach that transforms the long-term optimization problem into a single-time-slot problem, enabling effective handling of dynamic inference requests. Additionally, we introduce a coalition formation game model for user-server association, decoupling the complex joint optimization into manageable subproblems and enabling efficient suboptimal solutions via iterative optimization.
  \item We present a greedy algorithm for model deployment within each coalition, leveraging submodularity to provide efficient suboptimal solutions. We also develop an exhaustive search method to determine the optimal model partitioning strategy for each MD.
  \item We conduct extensive simulations to evaluate the performance of the proposed algorithms. Simulation results demonstrate that, compared to baseline algorithms, the proposed algorithms effectively minimize the average inference delay while satisfying long-term privacy constraints.
\end{itemize}

The remainder of this paper is organized as follows. Section 2 reviews related work on collaborative edge inference, model deployment, and privacy-aware DNN partitioning. Section 3 presents the system model, including the architecture of the edge-end collaborative inference framework, and formulates the long-term joint optimization problem aimed at minimizing average inference delay under resource and privacy constraints, and proves its NP-hardness. Section 4 presents the Lyapunov-based problem transformation, details the coalition formation game model, and describes the proposed algorithms for solving the optimization problem. Section 5 details the performance evaluation and discusses the simulation results. Finally, Section 6 concludes the paper and outlines potential directions for future work.

\section{Related Work}
Extensive research has been devoted to partition-based collaborative EI, wherein DNN models are strategically divided between edge servers and end-user devices to optimize computational load distribution and minimize communication overhead. For example, foundational works such as \cite{c21,c22} have established methodologies for model partitioning and inference optimization, while recent systematic surveys~\cite{c23,c24,c25} provide comprehensive overviews of foundational EI research. Aligned with the focus of this work, we review two critical dimensions of edge-end collaborative inference: 1) model deployment, placement, or caching strategies for efficient DNN inference, and 2) privacy-aware mechanisms to mitigate data leakage risks during collaborative inference. Representative works in these areas are discussed below.

\subsection{Model Deployment for Efficient DNN Inference}
Recent studies have made significant progress in optimizing DNN deployment across edge-cloud environments. In~\cite{c15}, a Lyapunov-based joint resource management scheme for IIoT systems is developed, co-optimizing model deployment, task offloading, and resource allocation to minimize delay and error penalties while ensuring system stability. JointDNN, proposed in~\cite{c26}, is a collaborative inference engine that employs shortest path optimization and integer linear programming to enhance performance and energy efficiency between mobile devices and cloud servers. For large-scale deployments,~\cite{c27} introduces a multi-paradigm deployment model that utilizes heuristic algorithms to optimize accuracy, scale, and cost for deep learning inference services. Resource-constrained scenarios are addressed in~\cite{c28}, which focuses on neural network optimization for edge devices by leveraging logic circuits, in-memory computing, and learning automata algorithms to improve efficiency and accuracy.

In the context of edge server deployments,~\cite{c29} establishes a joint caching and inference framework for foundation models, utilizing a least context algorithm to enhance inference accuracy and reduce system costs. Automation in cloud deployment is further advanced by the AutoDeep framework proposed in~\cite{c30}, which integrates Bayesian optimization and deep reinforcement learning to optimize cloud configurations and device placement, achieving significant cost reductions through techniques such as probing-informed block multiplexing. For heterogeneous edge environments, a joint optimization framework is designed in~\cite{c31} for device placement and model partitioning, employing evolutionary algorithms and dynamic programming to maximize throughput and minimize inference time.

Dynamic co-inference scenarios are further explored in~\cite{c32}, where model placement and online splitting in wireless networks are formulated as an optimal stopping problem, with algorithms proposed to minimize energy and time costs. In~\cite{c33}, multi-task inference in vehicular edge computing is addressed through a share-aware joint deployment and task offloading framework, utilizing a time period-aware algorithm to reduce total response time. Finally,~\cite{c34} investigates energy-optimal DNN placement in UAV-enabled edge networks, developing a Lyapunov-based online algorithm to minimize transmission delay and energy costs while stabilizing data queues.

In contrast to these studies that predominantly focus on isolated aspects of performance optimization, our work holistically investigates privacy-aware model deployment, adaptive model partitioning, and user–edge server association under resource and privacy constraints, a tripartite challenge unaddressed in existing literature, where privacy considerations are often oversimplified or decoupled from system-level coordination.
\subsection{Privacy-Aware Collaborative DNN Inference}
Recent advancements in privacy-preserving deep learning inference for edge and IoT systems address critical challenges in distributed intelligence. DistPrivacy~\cite{c35} introduces a distributed feature map allocation methodology that enhances data privacy for IoT surveillance systems through strategic data partitioning. In~\cite{c36}, an adaptive DNN partition framework is developed to dynamically balance privacy and performance under varying network conditions. The coalition-based approach in~\cite{c37} optimizes IoT camera–edge node associations, simultaneously addressing privacy preservation, energy efficiency, and multi-view detection performance through game-theoretic optimization. ~\cite{c38} proposes a quality-aware framework integrating adaptive model splitting and device association to handle resource and data heterogeneity in edge environments.

In device–edge co-inference scenarios,~\cite{c39} designs a Deep Reinforcement Learning (DRL)-based model splitting mechanism to achieve optimal privacy–computation trade-offs. Extending distributed inference capabilities,~\cite{c40} presents RL-DistPrivacy, employing reinforcement learning to optimize privacy-aware deep inference in latency-sensitive IoT systems.

Recent innovations further enhance both privacy and reliability:~\cite{c41} develops a queue-optimized CNN distributed inference method that leverages reinforcement learning to minimize latency while ensuring customer privacy protection. Finally,~\cite{c42} introduces Salted DNNs, a novel semantic rearrangement technique that uses cryptographic salt to enhance output privacy in split inference without compromising computational efficiency. Collectively, these works significantly advance the state-of-the-art in privacy-aware edge intelligence; however, they predominantly focus on isolated aspects of the privacy–performance trade-off, leaving opportunities for more holistic frameworks that integrate adaptive model deployment optimization, user–server association, and comprehensive privacy preservation mechanisms.

\section{System model and problem formulation}

\subsection{System Overview}
As shown in Fig. \ref{fig3}, we consider a hierarchical edge network comprising a cloud server, $M$ edge servers $\mathcal{M}=\{1,2,\dots,M\}$ (co-located with base stations), and $N$ resource‑constrained MDs $\mathcal{N}=\{1,2,\dots,N\}$. The cloud server maintains a library of $L$ distinct pre-trained DNN models $\mathcal{L}=\{1,2,\dots,L\}$. Each model $l \in \mathcal{L}$ provides a specific inference service to MDs, such as image classification, human motion detection, or autonomous driving. Each model $l$ requires storage resource $D_{l}$ for deployment and incurs a unit computation workload $W_{l}$ for model execution, measured in Floating Point Operations (FLOPs). Each edge server $m \in \mathcal{M}$ is equipped with limited computation resource $F_{m}$ (measured in Floating Point Operations Per Second, FLOPS) , storage resource $C_{m}$ (in GB), and spectrum bandwidth $B_m$ (in MHz). Each MD $n \in \mathcal{N}$ also has limited computation resource $F_{n}$. The edge servers are connected to the cloud server via the wired backhaul links, and communicate with MDs through wireless links. System time is slotted into discrete intervals $\mathcal{T}$, and the system is stable within each time slot.

At each time slot $t$, each MD $n \in \mathcal{N}$ generates an inference service request $r_{n}(t)$, which can be described by a tuple $\{p_{l,n}(t), d_{n,l}(t)\}$. $p_{l,n}(t)$ represents the probability of MD $n$ requesting model $l$ at time slot $t$, which is assumed to follow a known probability distribution, e.g., Zipf-like distribution. $d_{n,l}(t)$ is the input data size (i.e., number of images) for inference service $l$, and we assume that $d_{n,l}(t)$ follows a general random distribution within the range of $[d_{n,l}^{min}, d_{n,l}^{max}]$. This assumption is commonly adopted in prior works on edge inference scheduling\cite{c_RequestTSC,c_RequestInfocom}, where each MD engages in a single dominant inference task during a short scheduling interval.

Consequently, the proposed framework integrates several tightly coupled components to enable efficient and privacy-aware collaborative edge inference: (\textit{i}) model deployment at edge servers, (\textit{ii}) user–server association, and (\textit{iii}) model partitioning between MDs and edge servers. The workflow unfolds as follows. At the beginning of each time slot $t$, each MD generates an inference request, which is submitted to its associated edge server according to the current user–server association strategy. To minimize inference delay, edge servers pre-download and deploy models from the cloud server based on an optimized model deployment strategy, ensuring that popular and high-demand models are locally available. Upon receiving a request, the edge server determines the model partitioning strategy, splitting the DNN model into two sub-models: the device-side model is executed on the MD, while the server-side model is processed by the edge server. This design allows collaborative inference while balancing privacy, communication, and computation costs. The subsequent subsections detail the system modeling and optimization formulations that underpin this workflow.

\begin{figure}[t]
	\centering
	\includegraphics[width=2.5in]{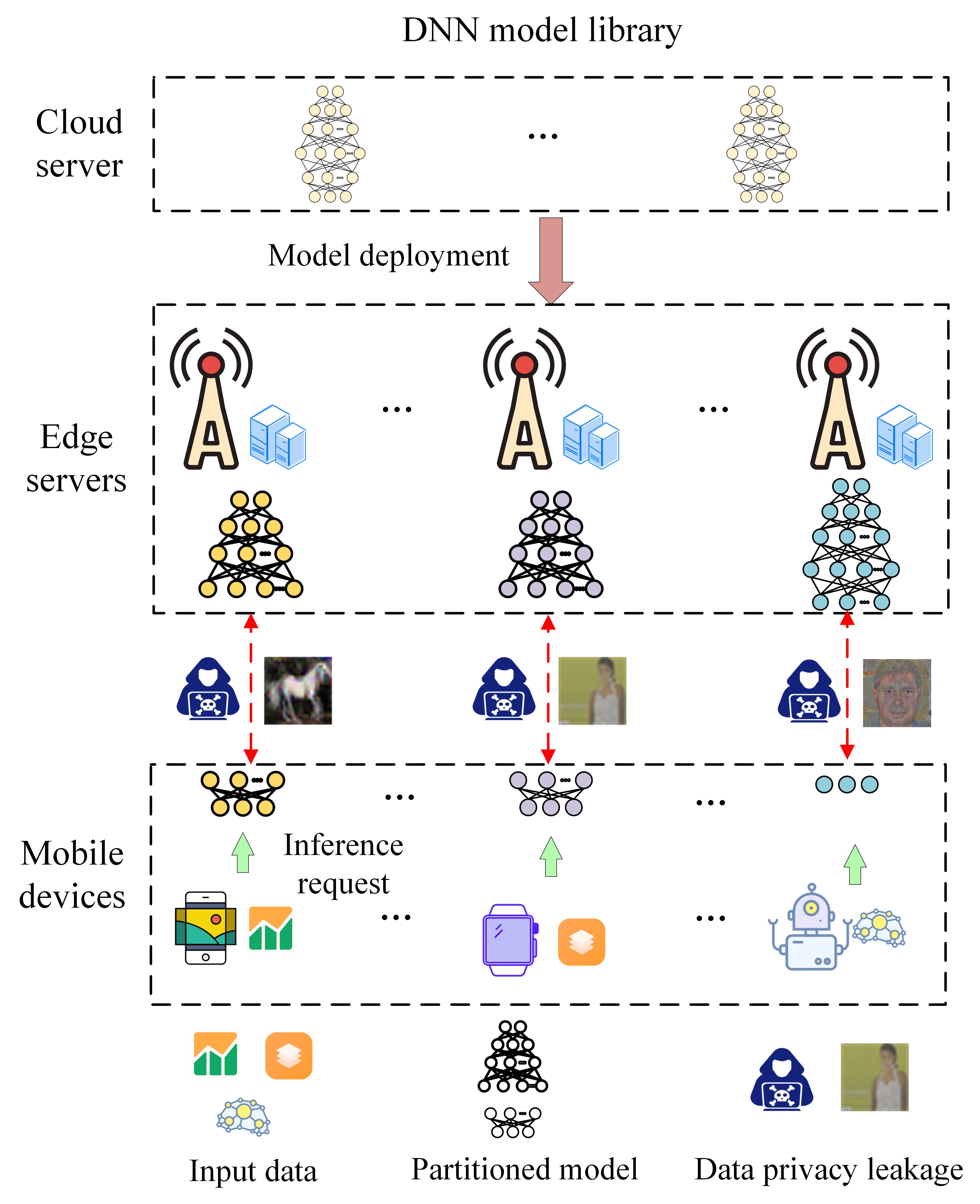}
	\caption{An illustration of edge-end collaborative EI in an edge computing network.}
	\label{fig3}
\end{figure}

\subsection{DNN Deployment and User Association Model}
Edge servers enable proximity-aware inference by locally deploying DNN models requested by MDs. However, limited storage prevents hosting the entire model library. Therefore, an effective deployment strategy is needed to decide which models to deploy on each server and when to update them based on changing demand. Let the binary deployment variable $x_{m,l}(t) \in \{0,1\},$ where $x_{m,l}(t) = 1$ if model $l$ is cached on edge server $m$ at time $t$, and $0$ otherwise.

Each edge server can store only a subset of models based on its storage capacity $C_m$, leading to the following constraint:
\begin{equation}
\text{C1:} \quad \sum_{l \in \mathcal{L}} x_{m,l}(t) D_l \leq C_m, \quad \forall m \in \mathcal{M},\ \forall t \in \mathcal{T}.
\end{equation}
Since servers may store different models and experience varying loads, user–server association critically affects performance. Each MD must connect to a server that both hosts the requested model and has sufficient computation and communication resources to minimize overall inference costs. We assume association decisions are made at the beginning of each time slot and remain fixed within the slot. This reduces synchronization overhead and ensures stable resource scheduling. To avoid the complexity of multi-server access, each MD can be associated with only one server per slot. Let $y_{m,n}(t) \in \{0,1\},$ where $y_{m,n}(t) = 1$ if MD $n$ is associated with edge server $m$ at time $t$, and $0$ otherwise.

The user-server association must satisfy the constraint:
\begin{equation}
\text{C2:} \quad \sum_{m \in \mathcal{M}} y_{m,n}(t) = 1, \quad \forall n \in \mathcal{N},\ \forall t \in \mathcal{T}.
\end{equation}

\subsection{DNN Partitioning Model}
In this work, we adopt a model partitioning strategy where each DNN model is split into a device-side and a server-side sub-model. The partition point affects both computation load on the MDs and servers, and the transmission volume between them. Let each DNN model $l \in \mathcal{L}$ consist of $K_l$ sequential layers. For a given partition point $z_{m,n,l}(t) \in \{0,1,\dots, K_l\}$, the first $z_{m,n,l}(t)$ layers are executed on the MD $n$, and the remaining $K_l-z_{m,n,l}(t)$ layers are processed by the associated edge server $m$. Additionally, two special cases are considered to enhance flexibility. If the MD has sufficient computational capability or strong privacy constraints, it can download the entire model and execute full local inference (i.e., $z_{m,n,l}(t)$=0). Conversely, if the MD has limited resources and lower privacy concerns, it may upload the entire input data to the edge server for fully remote inference (i.e., $z_{m,n,l}(t)=K_l$).

Accordingly, we define $w_{l}^{dev}(z_{m,n,l}(t))$ as the cumulative computational load of the first $z_{m,n,l}(t)$ layers, and let  $w_{l}^{edge}(z_{m,n,l}(t))$ denote the computational load of the remaining layers assigned to the edge server, i.e., $W_l=W_{l}^{dev}(z_{m,n,l}(t)+W_{l}^{edge}(z_{m,n,l}(t))$. Similarly, let $D^{\text{dev}}_l(z_{l,m,n}(t))$ denote the device-side sub-model size, and $D^{\text{edge}}_l(z_{l,m,n}(t))$ is the edge-side sub-model size, then we have $D_l=D^{\text{dev}}_l(z_{l,m,n}(t))+D^{\text{edge}}_l(z_{l,m,n}(t))$. Besides, $I_l(z_{m,n,l}(t))$ is the intermediate feature size output at the partition point, which must be uploaded to the edge server for continued processing.

\subsection{Privacy Loss Evaluation Model}
In collaborative inference systems employing model partitioning, privacy concerns arise because intermediate feature maps transmitted from user devices to edge servers may still retain substantial semantic and structural information. These intermediate representations can be exploited by adversaries to reconstruct the original input data, posing significant privacy risks. To quantify potential privacy leakage, various metrics have been used in prior works, such as PSNR, MSE, and SSIM as discussed in Introduction. Among these, SSIM has been widely adopted due to its strong correlation with human-perceived visual similarity, making it a practical tool to assess privacy leakage\cite{c37,c39}.

Since SSIM measures the similarity between the original input image and the reconstructed image derived from intermediate features, a higher SSIM value indicates that the reconstructed image closely resembles the original, suggesting a higher privacy risk. \textit{We acknowledge that SSIM has limitations in fully capturing semantic-level privacy risks. However, in this work, we choose to retain SSIM for the following reasons: First, the primary focus of this paper is not on advancing privacy evaluation theory or developing the most precise privacy metric. Instead, we focus on system-level optimization of model deployment, user–server association, and model partitioning under general privacy constraints; Second, the choice of specific privacy evaluation metrics does not fundamentally alter the optimization framework proposed in this work. Our framework is designed to be flexible and can incorporate alternative privacy metrics without loss of generality}.

SSIM quantifies image similarity by considering luminance, contrast, and structural information between two images $i$ and $j$. The original SSIM computation~\cite{wang2004SSIM} is defined as:
\begin{equation}
\text{SSIM}(i,j) = \frac{(2\mu_i \mu_j + C_1)(2\sigma_{ij} + C_2)}
{(\mu_i^2 + \mu_j^2 + C_1)(\sigma_i^2 + \sigma_j^2 + C_2)},
\end{equation}
where $\mu_i$ and $\mu_j$ are the means of images $i$ and $j$; $\sigma_i^2$ and $\sigma_j^2$ are their variances; $\sigma_{ij}$ is the covariance; and $C_1$, $C_2$ are small constants to stabilize the division. The theoretical range of SSIM is $[-1, 1]$, where 1 indicates perfect similarity and lower values indicate decreasing similarity.

For consistency in our privacy loss evaluation, we normalize the SSIM value to map it from $[-1,1]$ into $[0,1]$, using the transformation:
\begin{equation}
\phi(z_{m,n,l}(t)) = \frac{\text{SSIM}(z_{m,n,l}(t)) + 1}{2},
\end{equation}
Here, $\phi(z_{m,n,l}(t))$ denotes the privacy risk (or possibility) determined by model partition point $z_{m,n,l}(t)$.  $\text{SSIM}(z_{m,n,l}(t))$ is a shorthand notation indicating the SSIM computed between the original input image and the reconstructed image obtained by applying an adversarial reconstruction algorithm to the intermediate feature maps generated at the partition point $z_{m,n,l}(t)$. A higher $\phi(z_{m,n,l}(t))$ value implies a higher estimated privacy risk, i.e., a higher possibility of data privacy loss.

To simplify notation, we denote the estimated cumulative privacy loss by MD $n$ at time slot $t$ as $\Upsilon_n(t)$, which can be as \cite{c39}:
\begin{equation}
\Upsilon_n(t) = y_{m,n}(t) d_{n,l}(t)\ \phi(z_{m,n,l}(t))
\end{equation}

Based on the above discussion, it is necessary to conduct preliminary experiments on the dataset using the corresponding models to obtain empirical SSIM values. \textit{A key question is whether a correlation exists between the SSIM value and the partition point or model layer index}. If such a relationship can be approximated by a fitting function, it would enable efficient estimation of SSIM values at different partition points, thereby reducing the need for extensive pre-experiments. To explore this, we utilize a series of LeNet and VGG models, plotting scatter diagrams and fitting curves of SSIM values against the model layer index. The results are shown in Fig.~\ref{fig4}.

As illustrated in the figure, both the LeNet and VGG models exhibit similar trends: SSIM values decrease gradually with increasing model depth. Through curve fitting, these trends can be approximated by the following sigmoid-like function:
\begin{equation}\label{}
\phi(z_{m,n,l}(t)) = \frac{\omega_1}{1 + \exp\left(-\omega_2 (z_{m,n,l}(t) - \omega_3)\right)} + \omega_4
\end{equation}
where $\omega_1$, $\omega_2$, $\omega_3$, and $\omega_4$ are fitting constants. For example, we obtain $\omega_1=0.9031$, $\omega_2=-1.6683$, $\omega_3=4.9119$, $\omega_4=0.0983$ for LeNet models, and $\omega_1=0.6957$, $\omega_2=-0.6047$, $\omega_3=6.7718$, $\omega_4=0.3371$ for VGG models. While these parameter values may vary across different datasets, the general trend remains consistent.

It is worth noting that we were unable to achieve a good fit for ResNet models. This is because ResNet adopts a block-based architecture, and partitioning is performed at the block level, causing the SSIM values to drop sharply at certain points and then stabilize, as shown in Fig.~\ref{fig1}. Nevertheless, for other model families, if similar trends are observed through preliminary testing on a representative model on a specific data set, the fitted function can be applied to approximate the SSIM values of other variants in the same series. This approach can reduce the workload of pre-experimental SSIM evaluation to some extent.

\begin{figure}[t]
    \centering
    \subfigure[] {\includegraphics[width=1.71in,angle=0]{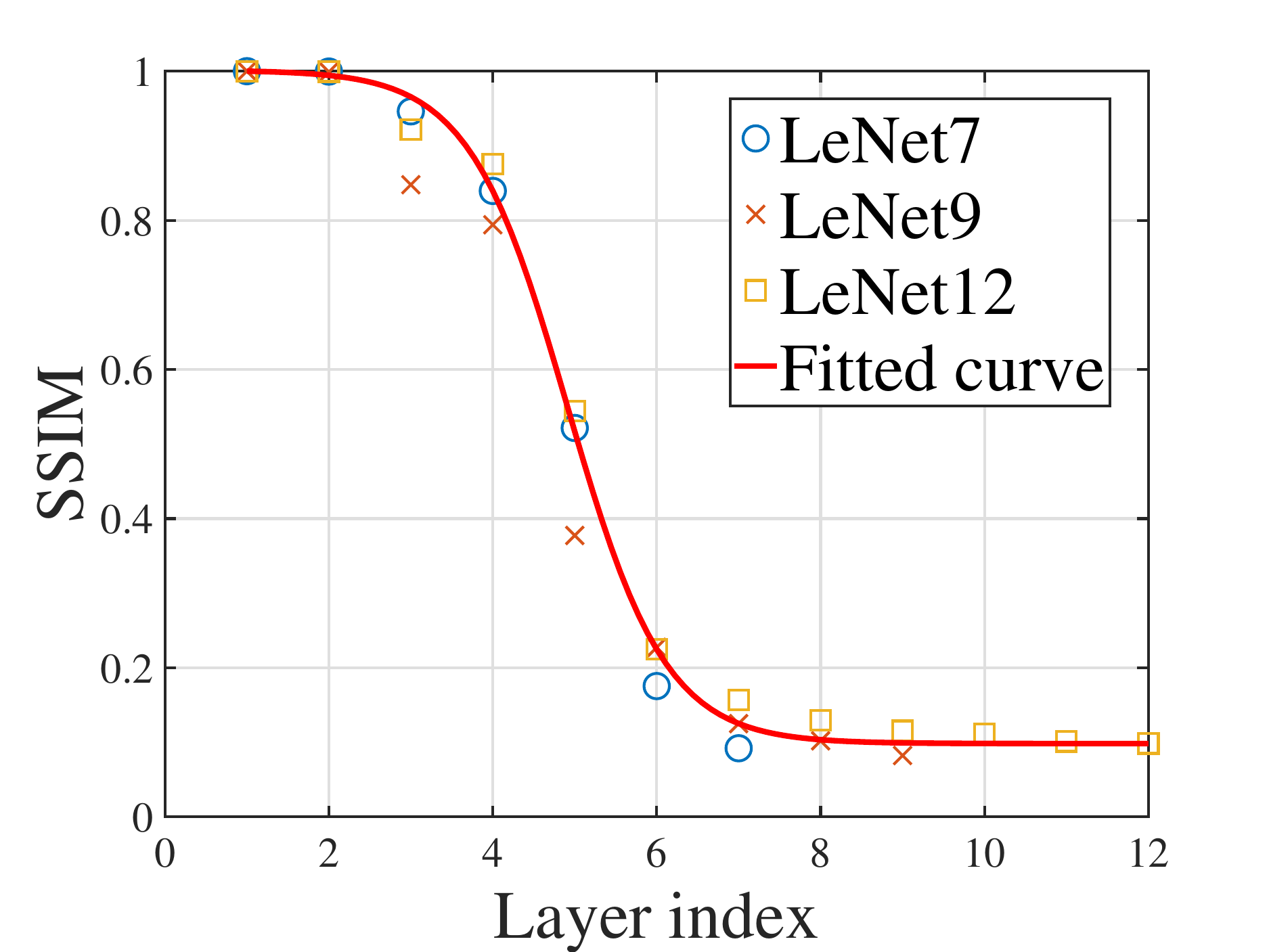}}
    \subfigure[] {\includegraphics[width=1.71in,angle=0]{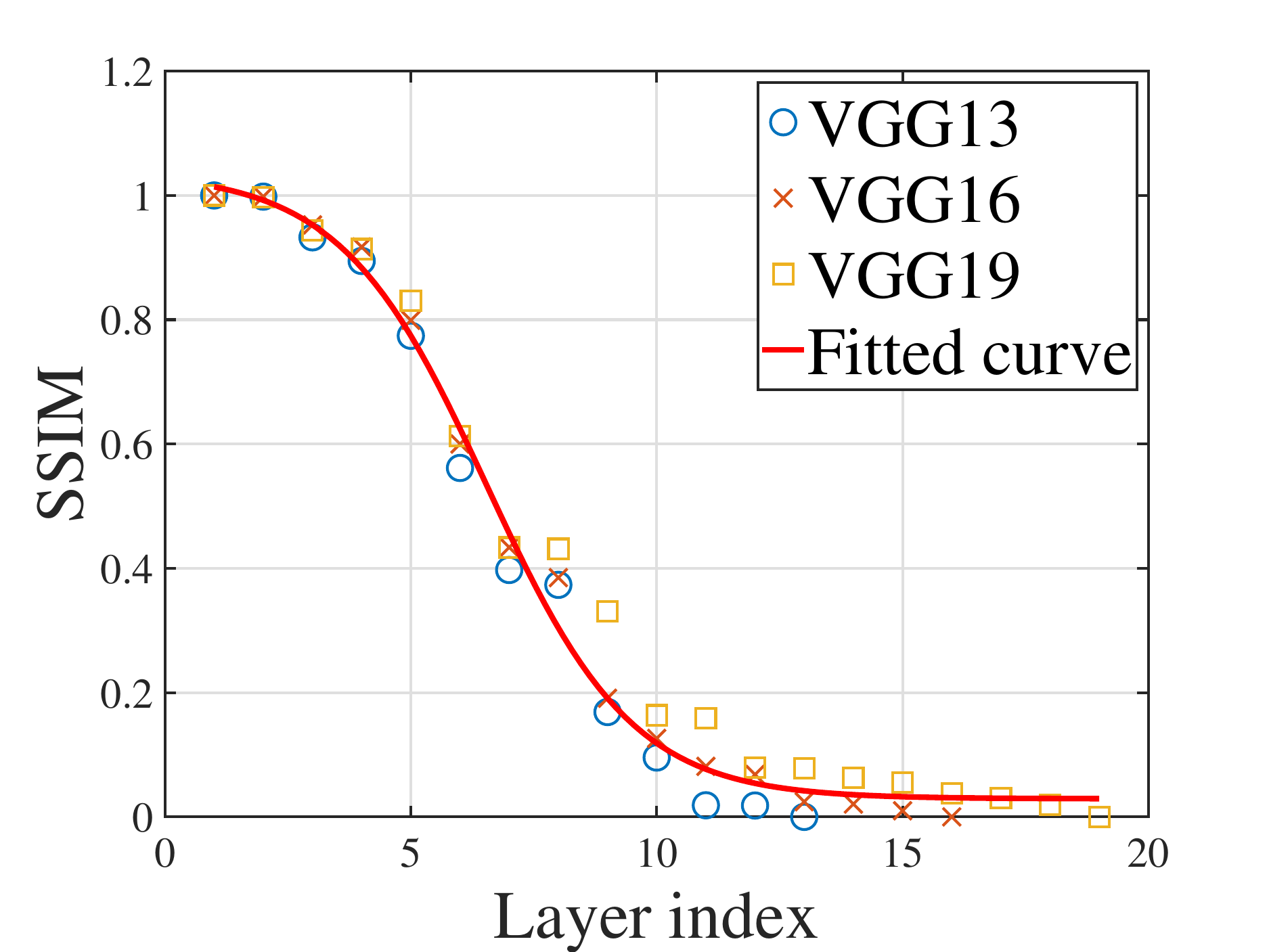}}
   \caption{The relationship between SSIM values and the model layer index: (a) LeNet models on CIFAR-10; (b) VGG models on Caltech-101.}
    \label{fig4}
\end{figure}

\subsection{Inference Delay Model}

The inference delay experienced by each MD depends jointly on the model deployment, user–server association, and model partitioning strategies. This section presents detailed modeling of the end-to-end delay across all stages of the collaborative inference process.

We assume that at the initial time slot $t=0$, no models are deployed at any edge server (i.e., $x_{m,l}(0)=0$ for all $l, m$). The full inference pipeline can be broken down as follows:

\textit{1) Cloud-to-Edge Model Transmission}: If the required model $l$ is not deployed on edge server $m$ at time $t-1$ but is deployed at time $t$ (i.e., $x_{m,l}(t)=1$ and $x_{m,l}(t-1)=0$), the server needs to download the full model from the cloud. The corresponding delay is:
\begin{equation}
\tau^{\text{c2e}}_{m,n,l}(t) = y_{m,n}(t) (1 - x_{m,l}(t-1)) x_{m,l}(t) \frac{D_l}{R_m^{c2e}},
\end{equation}
where $R_m^{c2e}$ is the cloud-to-edge backhaul transmission rate.

\textit{2) Device-side Sub-model Download}: Before local inference begins, MD $n$ needs to download the sub-model determined by the partition point $z_{l,m,n}(t)$ from its associated edge server. In this work, we assume the available bandwidth $B_{m}$ of each edge server is equally shared by its associated MDs. Thus, the bandwidth occupied by MD $n$ can be denoted as $b_{n,m}(t) = \frac{B_m}{N_m(t)}$, where $N_m(t) = \sum_{n=1}^{N} y_{m,n}(t) \neq 0$ is the total number of associated MDs.

Then, we adopt the Shannon capacity formula to compute the uplink and downlink data rates between MD $n$ and edge server $m$. Assuming symmetric bandwidth $b_{m,n}(t)$ and frequency-flat AWGN channels, the achievable transmission rates are expressed as follows:
\begin{equation}
R^{\text{up}}_{m,n}(t) = b_{m,n}(t)  \log_2\left(1 + \frac{P_n(t)h^{\text{up}}_{m,n}(t)}{N_0 b_{m,n}(t)}\right),
\end{equation}
\begin{equation}
R^{\text{down}}_{m,n}(t) = b_{m,n}(t) \log_2\left(1 + \frac{P_m(t) h^{\text{down}}_{m,n}(t)}{N_0 b_{m,n}(t)}\right),
\end{equation}
where $P_n(t)$ and $P_m(t)$ are the transmit power of MD $n$ and server $m$. $h^{\text{up}}_{m,n}(t)$ and $h^{\text{down}}_{m,n}(t)$ are the channel gains. $N_0$ is noise power spectral density.
Therefore, the corresponding download delay is:
\begin{equation}
\tau^{\text{down}}_{m,n,l}(t) = y_{m,n}(t) \frac{D^{\text{dev}}_l(z_{m,n,l}(t))}{R^{\text{down}}_{m,n}(t)}.
\end{equation}

\textit{3) Local Inference on MD}: After downloading completes, MD $n$ executes the device-side sub-model on its local input data. The local computation delay is:
\begin{equation}
\tau^{\text{local}}_{m,n,l}(t) =  y_{m,n}(t) d_{n,l}(t) \frac{W_l^{\text{dev}}(z_{m,n,l}(t))}{F_n}.
\end{equation}

\textit{4) Intermediate Feature Upload}: After local processing, the intermediate feature map generated by the partition point is uploaded to the edge server. The upload delay is:
\begin{equation}
\tau^{\text{up}}_{m,n,l}(t) = y_{m,n}(t) \frac{d_{n,l}(t)I_l(z_{l,m,n}(t))}{R_{m,n}^{\text{up}}(t)}.
\end{equation}

\textit{5) Inference on Edge}: Upon receiving the intermediate results, edge server $m$ continues processing the remaining layers of the model. The edge-side inference delay is:
\begin{equation}
\tau^{\text{edge}}_{m,n,l}(t) = y_{m,n}(t) \frac{d_{n,l}(t) W_l^{\text{edge}}(z_{m,n,l}(t))}{f_{n,m}(t)},
\end{equation}
where $f_{m,n}(t)=F_M/N_{m}(t)$ is the computation resource allocated to MD $n$ by server $m$ at time $t$.

Due to the typically small size of inference results, the transmission delay to return final results to MDs is considered negligible.

\textit{6) Total Inference Delay with Pipelined Execution}: Considering the possibility of pipelined execution, where MDs can upload intermediate features as soon as they are generated, and edge servers can begin processing received features immediately. To capture this fine-grained parallelism, we refine the total inference delay experienced by MD $n$ when requesting service $l$ with edge server $m$ at time $t$ as follows:
\begin{equation}
\begin{aligned}
\tau_{n}(t) =\
& \tau^{\text{c2e}}_{m,n,l}(t) + \tau^{\text{down}}_{m,n,l}(t) \\
& + \max \Big\{ \tau^{\text{local}}_{m,n,l}(t),\
   \max\big( \tau^{\text{up}}_{m,n,l}(t),\ \tau^{\text{edge}}_{m,n,l}(t) \big) \Big\},
\end{aligned}
\end{equation}
where $\tau_n(t)$ is used as a simplified notation for the total inference delay of MD $n$ when no ambiguity arises.

Finally, the total inference delay of all MDs can be denoted as:
\begin{equation}\label{}
\tau(t) = \sum_{n=1}^{N}\tau_{n}(t).
\end{equation}

\subsection{Problem Formulation and Hardness}
This paper aims to enhance the inference delay performance at the systemic level while simultaneously reducing the risk of inference data privacy leakage, thereby achieving a secure and efficient inference service framework. Furthermore, due to the stochastic and dynamic of inference service requests, optimizing any individual time slot alone cannot guarantee the long-term stability of the system. Therefore, this paper decides to minimize the long-term average inference delay under sustained privacy leakage constraints by jointly optimizing the model deployment strategy $\mathcal{X}(t)=\left\{x_{m,l}(t) \mid l \in \mathcal{L}, m \in \mathcal{M}\right\}$, association strategy $\mathcal{Y}(t)=\left\{y_{m,n}(t) \mid m \in \mathcal{M}, n \in \mathcal{N} \right\}$, and model partition strategy $\mathcal{Z}(t)=\left\{z_{m,n,l}(t) \mid l \in \mathcal{L}, m \in \mathcal{M}, n \in \mathcal{N} \right\}$. The formulated optimization problem can be expressed as follows:
\begin{equation}\label{p1}
\begin{aligned}
\mathcal{P}_1: & \min _{\mathcal{X}(t), \mathcal{Y}(t), \mathcal{Z}(t)}
\quad \lim_{T \rightarrow \infty} \frac{1}{T} \sum_{t=0}^{T-1} \mathbb{E}[\tau(t)] \\
\text{s.t.} \quad
& C1, C2, \\
& C3: x_{l,m}(t),\ y_{m,n}(t) \in \{0, 1\},\quad \forall l,m,n,t, \\
& C4: z_{l,m,n}(t) \in \{0, 1, \dots, K_l\},\quad \forall l,m,n,t, \\
& C5: \lim_{T \rightarrow \infty} \frac{1}{T} \sum_{t=0}^{T-1} \mathbb{E}[\Upsilon_n(t)] \leq \bar{\Upsilon}_n,\quad \forall n .
\end{aligned}
\end{equation}

In the above problem formulation, C5 ensures that the long-term average privacy loss for each MD does not exceed the specified threshold $\bar{\Upsilon}_n$. $\bar{\Upsilon}_n$ can reflect the MD's sensitivity to privacy loss.    The resulting problem $\mathcal{P}_1$ is a long-term stochastic optimization problem with combinatorial decision variables. It is challenging due to the joint optimization over discrete deployment and association variables, integer-valued partitioning strategies, and nonlinear effects introduced by the pipelined inference delay. In particular, even under a single-slot setting (ignoring C5 and time coupling), the problem remains NP-hard. This is because the impact of model partitioning on both latency and privacy cannot be captured by any simple convex or continuous mapping. We formalize this result in following Theorem 1 below.

\begin{myTheo}\label{theo1}
The proposed problem (\ref{p1}) is NP-hard in a single time slot.
\end{myTheo}

\begin{proof}
The proof can be found in Appendix A.
\end{proof}

Since Theorem \ref{theo1} can be regarded as a special case when $T=1$, we can also prove that problem (\ref{p1}) is NP-hard.

\section{LYAPUNOV-Based Problem Transformation, Decomposition, and Solution Design}
In this section, we initially transform the long-term optimization problem into a single-time-slot optimization problem based on Lyapunov optimization theory, thereby mitigating the impact of long-term constraint conditions. Subsequently, considering the distinct impacts of three different optimization strategies on the optimization objective, we decouple the problem for optimization. We first formulate the association problem between the MD and edge server as a coalition game problem. Then, we address the model deployment and partition optimization problem under a given coalition structure. Finally, through iterative optimization, we obtain an efficient suboptimal solution for the original problem, and analyze the properties of the proposed algorithms.

\subsection{Lyapunov-Based Problem Transformation}
In problem (\ref{p1}), constraint C5 is used to stabilize the time-averaged maximum data privacy leakage. However, this approach requires complete information of the system for every time slot, which is not feasible to acquire in advance in real-world systems. To address this issue, we introduce accumulated data privacy leakage in Definition 1.
\begin{myDef}\label{def1}
Let $\Xi(t+1)$ denote the accumulated data privacy leakage that exceeds the data privacy threshold over $t$ time slots, which can be calculated as
\begin{equation}\label{}
\Xi(t+1) =\left[\Xi(t)+ \sum_{n=1}^{N}\left(\Upsilon_n(t)-\bar{\Upsilon}_n \right)\right]^{+}.
\end{equation}
\end{myDef}

 $\Xi(0)=0$ denote the initial value of the accumulated data privacy leakage. A higher value of $\Xi(t)$ indicates that the data privacy leakage caused by inappropriate strategy-making exceeds the time-averaged privacy constraint. Based on Definition \ref{def1}, we can transform the long-term constraint C7 in problem (\ref{p1}) as follows
\begin{equation}\label{}
\lim _{T \rightarrow \infty} \frac{1}{T} \sum_{t=0}^{T-1} \mathbb{E}\left[\Xi(t)\right] \leq 0 .
\end{equation}

Then, we introduce the most widely applied Lyapunov function ${L}(\Xi(t)) \triangleq \frac{1}{2} \Xi^2(t)$, to measure the satisfaction status of the long-term privacy constraint. Thus, a smaller value of ${L}(\Xi(t))$ means a better data privacy satisfaction. To ensure the value of ${L}(\Xi(t))$ is small enough to meet long-term privacy constraints, we define $\Delta(\Xi(t))$ below by introducing Lyapunov drift function:
\begin{equation}\label{ldrift}
\Delta(\Xi(t)) \triangleq \mathbb{E}[L(\Xi(t+1))-L(\Xi(t)) \mid \Xi(t)].
\end{equation}
By expanding the above expression further, we can obtain that:
\begin{equation}\label{drift}
\begin{aligned} & \Delta(\Xi(t))=\frac{1}{2} \mathbb{E}\left[\Xi^2(t+1)-\Xi^2(t) \mid \Xi(t)\right] \\ & =\mathbb{E}\left[\Xi(t) \sum_{n=1}^{N}\left(\Upsilon_n(t)-\bar{\Upsilon}_n \right) \mid \Xi(t)\right]+ \\& \frac{1}{2} \mathbb{E}\left[ \sum_{n=1}^{N}\left(\Upsilon_n(t)-\bar{\Upsilon}_n \right)^2 \mid \Xi(t)\right].\end{aligned}
\end{equation}

Next, within the Lyapunov drift-penalty framework, we incorporate equation (\ref{drift}) into the optimization objective to strike a balance between minimizing inference delay and ensuring long-term privacy constraints, as shown below:
\begin{equation}\label{newobj}
\mathbb{E}\left[\tau(t)\mid \Xi(t) \right]+\Delta\left(\Xi(t)\right).
\end{equation}

According to (\ref{ldrift}), the calculation of $\Delta(\Xi(t))$ requires the information of $L(\Xi(t+1))$, which will not be available in each time slot $t$. Thus, we tend to obtain the upper bound of (\ref{newobj}) that can be calculated only based on the available information in time slot $t$. Firstly, it can be easily deduced that $\Upsilon_{n}(t)$ will be no lower than 0, and $\sum_{n=1}^{N}\left(\Upsilon_n(t)-\bar{\Upsilon}_n \right)$ will be no higher than $\sum_{n=1}^{N}\bar{\Upsilon}_{n}^2$. Then, define a constant $\Theta =\frac{1}{2}\sum_{n=1}^{N}\bar{\Upsilon}_{n}^2$, the upper bound of $\Delta\left(\Xi(t)\right)$ can be represented as:
\begin{equation}\label{upper}
\Delta\left(\Xi(t)\right) \leq \Xi(t) \cdot \mathbb{E}\left[\sum_{n=1}^{N}\left(\Upsilon_n(t)-\bar{\Upsilon}_n\right) \mid \Xi(t)\right]+\Theta
\end{equation}

Based on (\ref{upper}), the upper bound of (\ref{newobj}) can be deduced as:
\begin{equation}\label{upperbound}
\begin{aligned}
& \mathbb{E}\left[\tau(t)\mid \Xi(t) \right]+\Delta\left(\Xi(t)\right) \leq \\
&  \mathbb{E}\left[\tau(t)\mid \Xi(t) \right]-\Xi(t) \cdot \mathbb{E}\left[\sum_{n=1}^{N}\left(\bar{\Upsilon}_n-{\Upsilon}_n(t)\right) \mid \Xi(t)\right]+\Theta.
\end{aligned}
\end{equation}

Based on (\ref{upperbound}), Problem $\mathcal{P}_1$ can be approximated only rely on the information in each current time slot. Thus, it can be solved by finding the solution to the following problem $\mathcal{P}_2$ that minimize the upper bound in each time slot over $T$:
\begin{equation}\label{p2}
\begin{aligned}
& \mathcal{P}_2: \min _{\mathcal{X}(t), \mathcal{Y}(t), \mathcal{Z}(t)} \alpha \cdot \tau(t)-\Xi(t) \cdot \sum_{n=1}^{N}\left(\bar{\Upsilon}_n-{\Upsilon}_n(t)\right)+\Theta\\
& \text { s.t. } \quad \text { Constraints : }\left(C_1\right)-\left(C_4\right).
\end{aligned}
\end{equation}
where $\alpha$ is a positive constant that is used to achieve the tradeoff between delay minimization and satisfaction status of the long-term privacy constraints.

\subsection{Coalition Formation Game Model for Edge Server Association}
Using Lyapunov optimization, we reformulated the original long-term problem into a per-slot optimization. However, this per-slot problem remains NP-hard due to the combinatorial nature of the user–server association. Specifically, assigning $N$ MDs to $M$ edge servers results in a search space of size $N^M$, which is computationally intractable for large-scale systems. Moreover, user association is not a purely local decision. MDs sharing the same edge server must compete for limited wireless bandwidth and computation resources, leading to tightly coupled inference performance among them. This interdependence implies that which MDs associate with the same server significantly affects individual and collective inference delay. Hence, the association strategy should be jointly optimized among MDs who potentially share a server.

To efficiently address this problem structure, we adopt a coalition formation game framework. In this setting, each edge server and its associated MDs are viewed as a coalition, where MDs benefit from cooperating by associating with the same server while jointly managing shared resources. Compared to other game-theoretic models (e.g., non-cooperative games or auction-based schemes), coalition games naturally capture both the collaborative resource allocation and the dynamic group formation process. This allows us to design distributed algorithms where MDs autonomously form stable coalitions with edge servers based on local utility improvements, while significantly reducing the overall search complexity\cite{c37, xiaTPDS}. Therefore, we first model the association problem as a cooperative coalition formation game. Once stable coalitions are formed, we optimize model deployment and partition strategies within each coalition independently, following the decomposition logic. This layered approach enables scalable, decentralized decision-making while capturing key resource coupling and collaboration characteristics inherent to edge intelligence systems.

\subsubsection{Coalition Formation Game Model}
Under any given model deployment and model partition strategies, associating $N$ MDs with $M$ edge servers can be regarded as a coalition game involving $M$ coalitions. Initially, we provide the fundamental definitions concerning coalition games.
\begin{myDef}\label{def2}
The MD-server association coalition formation game is denoted by $\left(\mathcal{N}, U, \mathcal{F}\right)$, where $\mathcal{N}$ denote the $N$ MDs as the players, and $U$ is a mapping that determines the utility of the coalitions. $\mathcal{F}=\{\mathcal{F}_1,\mathcal{F}_2,\dots, \mathcal{F}_M\}$ denote the coalition set of $M$ mutually disjoint coalitions, where $\cap_{m=1}^M \mathcal{F}_m=\emptyset$ and $\cup_{m=1}^M \mathcal{F}_m=\mathcal{N}$. The strategy of each player is to joint or leave the coalition to improve the coalition utility, or exchange to improve the total utility.
\end{myDef}

\begin{myDef}\label{def3}
For coalition $F_m$ in coalition structure $\mathcal{F}$, we define the following function to characterize the coalition utility of $\mathcal{F}_m$:
\begin{equation}\label{}
U(\mathcal{F}_m) = \frac{-\alpha \sum \limits_{n \in F_m} \tau_n+\Xi(t) \cdot \sum \limits_{n \in \mathcal{F}_m}\left(\bar{\Upsilon}_n-{\Upsilon}_n(t)\right)-\Theta}{|\mathcal{F}_m|}
\end{equation}
where $|\mathcal{F}_m|$ denote the number of MDs in $\mathcal{F}_m$.
\end{myDef}

Based on Definition \ref{def3}, each coalition aims to minimize the average inference delay and privacy leakage by allocating the MDs into different coalitions, which also aligns with the original optimization objective. For any partition of $\mathcal{F}=\{\mathcal{F}_1,\mathcal{F}_2,\dots, \mathcal{F}_M\}$, the social welfare of the game is
\begin{equation}\label{}
U(\mathcal{F}) = \sum_{m=1}^{M} U(\mathcal{F}_m)
\end{equation}

To ensure that players can join a coalition that yields higher utility, it is necessary to define the players' preferences for different coalitions. Thus, we define the preference order for player $n$.
\begin{myDef}
For any player $n \in \mathcal{N}$, the preference order $\succ_{n}$ is defined as a complete, transitive, and reflexive binary relation over the set of all partitions that player $n$ can possibly form. Give two partitions $\mathcal{F}$ and $\mathcal{F}^{\prime}$, we say player $n$ prefers $\mathcal{F}$ over $\mathcal{F}^{\prime}$ if and only if:
\begin{equation}\label{}
\mathcal{F} \succ_{n} \mathcal{F}^{\prime} \Leftrightarrow U(\mathcal{F}) > U(\mathcal{F}^{\prime})
\end{equation}
\end{myDef}

Based on the above definition, players establish preferences across all possible coalitions, and each player selects the coalition that enhances the total utility of the newly formed coalitions according to the preference order. Thus, we further define the switch operation for forming the final coalitions as follows.
\begin{myDef}
Switch operation: Given a partition $\mathcal{F}=\{\mathcal{F}_1,\mathcal{F}_2,\dots, \mathcal{F}_M\}$ of $\mathcal{N}$, for player $n \in \mathcal{F}_m$, when and only when $\mathcal{F}_{m^{\prime}} \succ_{n} \mathcal{F}_m$ ($m^{\prime}\neq m$) is achieved, a switch operation moves player $n$ from $\mathcal{F}_m$ to $\mathcal{F}_{m^{\prime}}$, then $\mathcal{F}_m$ will be replaced by $\mathcal{F}_{m^{\prime}}$, where $\mathcal{F}_{m^{\prime}}=(\mathcal{F}\setminus\{\mathcal{F}_m, \mathcal{F}_{m^{\prime}}\})\cup\{ \mathcal{F}_m\setminus\{n\},\mathcal{F}_{m^{\prime}}\cup \{n\} \}$.
\end{myDef}

In coalition formation game, the coalition structure may trap in local optimum if only single-player switch operation is conducted. Therefore, in addition to the operation where a single player leaves to join another coalition, players from two distinct coalitions can also directly engage in exchange operations to achieve coalition updates, thereby attaining greater total utility. Thus, we define the coalition exchange operation as follows.
\begin{myDef}
Exchange operation: Given a partition $\mathcal{F}=\{\mathcal{F}_1,\mathcal{F}_2,\dots, \mathcal{F}_M\}$ of $\mathcal{N}$, two different players $n \in \mathcal{F}_{m}$, and $ n^{\prime} \in \mathcal{F}_{m^{\prime}}$, they perform exchange operations, i.e., $n$ moves from $\mathcal{F}_{m}$ to $\mathcal{F}_{m^{\prime}}$, $ n^{\prime}$ moves from $\mathcal{F}_{m^{\prime}}$ to $\mathcal{F}_{m}$, when and only when $\mathcal{F}_{m^{\prime}} \succ_{n,n^{\prime}} \mathcal{F}_m$, and the partition $\mathcal{F}$ is replaced by $\mathcal{F}^{\prime}$, where $\mathcal{F}^{\prime}=\mathcal{F} \backslash\left\{F_m, F_{m^{\prime}}\right\} \cup\left\{F_m \backslash\{n\} \cup\left\{n^{\prime}\right\}, F_{m^{\prime}} \backslash\left\{n^{\prime}\right\} \cup\{n\}\right\}$

\end{myDef}

\subsubsection{Algorithm for Coalition Formation Game}

Following the definitions of coalition switching and coalition exchange previously outlined, MDs are able to iteratively modify their associated edge servers through a sequence of steps. Within each step, only one MD can switch its associated server or two MDs in different coalitions can exchange their servers. After a coalition change, the joint model deployment and partition strategies are re-optimized only for the affected coalitions using the proposed Algorithm 2 and Algorithm 3, respectively. This localized update reduces unnecessary computation and maintains the scalability of the algorithm. However, due to the complexity introduced by the switch operation, the proposed algorithm performs an exchange operation only after $G$ switch operations. The value of $G$ is proportional to the problem scale. When $M$ and $N$ are large, a larger $G$ can effectively reduce the number of exchange operations, thereby lowering the algorithm's time complexity. In some cases, exchange operations can even be entirely omitted, allowing the algorithm to efficiently obtain a suboptimal solution. By repeating the above steps until the final partition structure reaches stable, we summarize the proposed distributed algorithm in Algorithm 1. The convergence, stability and complexity of the algorithm is analyzed in the next subsection.

\begin{algorithm}[!htbp]\label{Alg1}
	\caption{Coalition Formation Game-Based Algorithm}
	\begin{algorithmic}[1]
    \State Initializes a random partition $\mathcal{F}_{init}=\{\mathcal{F}_1,\mathcal{F}_2,\dots, \mathcal{F}_M\}$: Each MD associate the edge server based on the given partition;
    \State Set the current partition $\mathcal{F}_{cur}$ as $\mathcal{F}_{cur} = \mathcal{F}_{init}$;
    \State Set number of iteration $t=1$;
    \Repeat
    \State Randomly select a MD $n$ and its coalition $\mathcal{F}_{m} \in \mathcal{F}_{cur}$, and randomly select another coalition $\mathcal{F}_{m^{\prime}} \in \mathcal{F}_{cur}$, $m\neq m^{\prime}$;
    \State Set a temp partition $\mathcal{F}_{temp}=\{\mathcal{F}_{m}\setminus\{n\}, \mathcal{F}_{m^{\prime}}\cup\{n\}\}$;
    \State Obtain the new partition $\mathcal{F}_{new}=(\mathcal{F}_{cur}\setminus\{\mathcal{F}_{m}, \mathcal{F}_{m^{\prime}}\})\cup \mathcal{F}_{temp}$;
    \State Optimize the model deployment strategy via \textbf{Algorithm 2}, and model partition strategy via \textbf{Algorithm 3} \textbf{for the two affected coalitions} in $\mathcal{F}_{cur}$ and $\mathcal{F}_{new}$;
    \State Calculate the coalition utility of $\mathcal{F}_{cur}$ and $\mathcal{F}_{new}$;
    \If{$\mathcal{F}_{new} \succ_{n}  \mathcal{F}_{cur}$}
        \State $\mathcal{F}_{cur} =\mathcal{F}_{new}$ ;
    \EndIf
    \If{$t\% G=0$}
        \State Randomly select two MDs $n$, $n^{\prime}$ from two different coalitions $\mathcal{F}_{m}$ and $\mathcal{F}_{m^{\prime}}$;
        \State Set a temp partition $\mathcal{F}_{temp}=\{(\mathcal{F}_m\setminus \{n\}) \cup \{n^{\prime}\},\ (\mathcal{F}_{m^{\prime}} \setminus \left\{n^{\prime}\right\}) \cup\{n\}\}$;
        \State Obtain the new partition $\mathcal{F}_{new}=(\mathcal{F}_{cur}\setminus\{\mathcal{F}_{m}, \mathcal{F}_{m^{\prime}}\})\cup \mathcal{F}_{temp}$;
        \State Optimize the model deployment strategy via \textbf{Algorithm 2}, and model partition strategy via \textbf{Algorithm 3} \textbf{for the two affected coalitions} in $\mathcal{F}_{cur}$ and $\mathcal{F}_{new}$;
        \State Calculate the coalition utility of $\mathcal{F}_{cur}$ and $\mathcal{F}_{new}$;
        \If{$\mathcal{F}_{new} \succ_{n,n^{\prime}}  \mathcal{F}_{cur}$}
            \State $\mathcal{F}_{cur} =\mathcal{F}_{new}$ ;
        \EndIf
    \EndIf
    \State $t=t+1$;
    \Until{The final partition reaches Nash-stable.}
	\end{algorithmic}
\end{algorithm}

\subsection{Problem Decomposition within Each Coalition}

After obtaining a given coalition structure, that is, the association pattern between MDs and edge servers, it is necessary to optimize the model deployment strategy and model partition strategy to maximize coalition utility. Given a fixed coalition (i.e., an edge server and its associated MDs), the joint optimization problem can be decomposed into two subproblems: (i) the model deployment subproblem, which can be formulated as a submodular maximization under storage constraint; and (ii) the model partition subproblem, solvable via exhaustive search over a finite set of partition candidates.

\subsubsection{Greedy-Based Algorithm for Model Deployment Problem}

Under the given edge association and model partition strategies, the optimization goal of the model deployment problem is to minimize the average delay and privacy loss of all MDs under limited model storage space. This problem has been proven to be NP-hard according to Theorem~\ref{theo1}. To derive an efficient suboptimal solution, we first establish the following property:

\begin{myTheo}\label{theo3}
The model deployment subproblem is a submodular maximization problem within each coalition under a given model partition strategy.
\end{myTheo}

\begin{proof}
The proof can be found in Appendix B.
\end{proof}

This submodularity arises from the diminishing marginal gain of deploying additional models as the server's limited storage becomes saturated. Based on this, we propose a greedy-based algorithm, summarized in Algorithm~2, which guarantees a \((1 - \frac{1}{e})\)-approximation ratio~\cite{xiaTPDS}. The key idea is to calculate a weighted utility-to-storage ratio \(\varpi_j\) for each model \(j\), capturing its potential utility gain normalized by storage cost. The algorithm iteratively selects the most cost-effective model until the server's storage capacity is exhausted.
\begin{algorithm}[t]\label{Alg-2}
	\caption{Greedy-Based Algorithm for Model Deployment}
	\begin{algorithmic}[1]
    \Require $\mathcal{L}$, $\mathcal{F}_m$, $C_m$, $\{D_l\}$, $\{p_{l,n}\}$
    \Ensure $\{x_{l,m}\}$
    \State Initialize deployed set $V = \emptyset$, utility $f(V)=0$, remaining storage $C^{\prime}=C_m$;
    \While{$C^{\prime} > 0$ and $\mathcal{L} \setminus V \neq \emptyset$}
        \For{each $j \in \mathcal{L} \setminus V$}
           \State Calculate aggregated request probability: $p_j = \sum_{n \in \mathcal{F}_m} p_{j,n}$;
           \State Compute marginal gain: $\Delta f(V, j) = f(V \cup \{j\}) - f(V)$;
           \State Compute cost-effectiveness: $\varpi_j = p_j \cdot \Delta f(V, j) / D_j$;
        \EndFor
        \State Select $j^* = \arg \max_{j \in \mathcal{L} \setminus V,\ D_j \leq C'} \varpi_j$;
        \If{$j^*$ exists}
            \State $V = V \cup \{j^*\}$;
            \State $f(V) = f(V \cup \{j^*\})$;
            \State $C' = C' - D_{j^*}$;
        \Else
            \State Break;
        \EndIf
     \EndWhile
     \State Set $x_{l,m}=1$ if $l \in V$, otherwise $x_{l,m}=0$.
	\end{algorithmic}
\end{algorithm}

For Algorithm~2, the model deployment subproblem is solved using a greedy heuristic over the submodular objective. In each iteration, the algorithm evaluates the utility gain of adding each undeployed model and selects the one with the highest weighted utility-to-storage ratio. As there are at most $L$ models, and each iteration performs up to $L$ evaluations, the overall time complexity is $\mathcal{O}(L^2)$.
\subsubsection{Optimal Algorithm for Model Partition Problem}

After obtaining the model deployment strategy within each coalition, the next step is to determine the optimal model partition strategy. For each MD, the goal is to select a model partition point such that the coalition utility is maximized (i.e., the combined delay and privacy loss is minimized). Since the number of partition layers in a DNN is finite and usually moderate, we adopt an exhaustive search strategy to find the optimal partition layer for each deployed model.
\begin{algorithm}[t]\label{Alg-3}
	\caption{Optimal Algorithm for Model Partition}
	\begin{algorithmic}[1]
    \Require $L$, $\mathcal{F}_m$, $\{x_{m,l}\}$, $\{\mathcal{A}_l\}$, $\{D_l\}$
    \Ensure $\{z_{m,n,l}\}$
    \For{each $l \in \mathcal{L}$ such that $x_{l,m}=1$}
        \For{each $n \in \mathcal{F}_m$}
            \State Initialize: $z^*_{m,n,l} \leftarrow 0$, $Obj \leftarrow +\infty$;
            \For{each partition index $z \in \mathcal{A}_l$}
                \State Compute objective $Obj(z)$ of $\mathcal{P}_2$ for $z_{m,n,l} = z$;
                \If{$Obj(z) < Obj$}
                    \State $Obj \leftarrow Obj(z)$,\quad $z^*_{m,n,l} \leftarrow z$;
                \EndIf
            \EndFor
            \State Save $z^*_{m,n,l}$ to output set $\{z_{m,n,l}\}$;
        \EndFor
    \EndFor
	\end{algorithmic}
\end{algorithm}

For Algorithm~3, the model partition strategy is determined by exhaustively searching over all candidate partition points for each MD–model pair. Suppose the maximum number of partition layers for any model is $K$, and the coalition contains $N_m$ MDs and $L_m$ deployed models. Then, the total complexity is $\mathcal{O}(N_m L_m A)$. Notably, each MD's partition decision is independent of others, and therefore this process can be executed in parallel across MDs, leading to significant runtime reduction in practical deployment.

\subsection{Algorithm Property Analysis}
In this part, we analyze the convergence, stability and complexity of the proposed
We first analyze the convergence of the proposed algorithm.
\begin{myTheo}\label{theo3}
The proposed coalition formation game-based algorithm in Algorithm 1 always converges to a final partition $\mathcal{F}_{final}$ with regardless the initial partition structure.
\end{myTheo}
\begin{proof}
The proof can be found in Appendix C.
\end{proof}

Then, we analyze the stability of the proposed algorithm, before that, we give the definition of $\mathbb{D}$-stability.
$\mathbb{D}$-stability (Deviation Stability) is a strong notion of coalition stability, ensuring that no group of players has an incentive to reorganize into a different coalition structure that yields a strictly higher total utility (social welfare)~\cite{c37}. We begin with the formal definition:
\begin{myDef}
A partition $\mathcal{F} = \{\mathcal{F}_1, \mathcal{F}_2, \dots, \mathcal{F}_M\}$ is said to be $\mathbb{D}$-stable if, for any possible deviation by a subset of players resulting in a new partition $\mathcal{F}'$, the total utility does not improve:
\begin{equation}
\sum_{\mathcal{F}_i \in \mathcal{F}'} U(\mathcal{F}_i) \leq \sum_{\mathcal{F}_i \in \mathcal{F}} U(\mathcal{F}_i)
\end{equation}
\end{myDef}
In our algorithm, only two types of deviations are permitted: (i) single-player switch operations and (ii) two-player exchange operations, both strictly increasing the total utility. Based on this, we establish the following result:
\begin{myTheo}\label{theo3}
The final converged partition $\mathcal{F}_{\text{final}}$ obtained by Algorithm~1 is $\mathbb{D}$-stable with respect to the defined deviation operations.
\end{myTheo}
\begin{proof}
The proof can be found in Appendix D.
\end{proof}

In Algorithm~1, each switch or exchange operation involves computing the coalition utility and executing Algorithm~2 and Algorithm~3. The complexity of coalition utility evaluation is $\mathcal{O}(M)$, while Algorithm~2 and Algorithm~3 respectively incur $\mathcal{O}(L^2)$ and $\mathcal{O}(NK)$ complexity.
\begin{table*}[t]
    \centering
    \small
    \setlength{\tabcolsep}{1pt}
    \renewcommand{\arraystretch}{1.1}
   \caption{Layer-wise Data Size, Computation (Comp.) Load, Communication (Comm.) Load, and Privacy Loss Possibility of VGG19}
    \label{tab:layer_params}
    \begin{tabular}{|l |c |c |c |c | l |c| c |c| c|}
        \toprule
       \textbf{Layer} & \textbf{Data Size} & \textbf{Comp. Load} & \textbf{Comm. Load } & \textbf{Possibility} &
        \textbf{Layer} & \textbf{Data Size} & \textbf{Comp. Load} & \textbf{Comm. Load} & \textbf{Possibility} \\ \hline
        & (KB) & (MFLOPs) & (KB) & & & (KB) & (MFLOPs) & (KB) &  \\ \hline
        \midrule
        Conv1\_1  & 7       & 86.7    & 12544  & 1.000  & Conv4\_2  & 9218    & 1849.7 & 1568  & 0.1636  \\ \hline
        Conv1\_2  & 144.3   & 1849.7  & 12544  & 0.9973 & Conv4\_3  & 9218    & 1849.7 & 1568  & 0.1586  \\\hline
        Conv2\_1  & 288.5   & 924.8   & 6272   & 0.9439 & Conv4\_4  & 9218    & 1849.7 & 1568  & 0.0793  \\\hline
        Conv2\_2  & 576.5   & 1849.7  & 6272   & 0.9137 & Conv5\_1  & 9218    & 462.4  & 392   & 0.0785  \\\hline
        Conv3\_1  & 1153    & 924.8   & 3136   & 0.8308 & Conv5\_2  & 9218    & 462.4  & 392   & 0.0627  \\\hline
        Conv3\_2  & 2305    & 1849.7  & 3136   & 0.6128 & Conv5\_3  & 9218    & 462.4  & 392   & 0.0551  \\\hline
        Conv3\_3  & 2305    & 1849.7  & 3136   & 0.4335 & Conv5\_4  & 9218    & 462.4  & 392   & 0.0381  \\\hline
        Conv3\_4  & 2305    & 1849.7  & 3136   & 0.4314 & FC1       & 401424  & 102.8  & 16    & 0.0300  \\\hline
        Conv4\_1  & 4610    & 924.8   & 1568   & 0.3311 & FC2       & 65552   & 16.8   & 16    & 0.0191  \\\hline
        -         & -       & -       & -      & -      & FC3       & 1616.4  & 0.41   & 0.4     & 0.0000  \\\hline
        \bottomrule
    \end{tabular}
\end{table*}
Assuming Algorithm~1 converges after $T$ total iterations, with one exchange operation every $G$ switches, the overall time complexity is:
\begin{equation}
\mathcal{O}\left( \left(1 + \frac{1}{G} \right) T \cdot (L^2 + NK + M) \right)
\end{equation}

Since all terms grow polynomially with the input size, the algorithm provides an efficient and scalable solution in practice.

\section{Simulation Evaluation}

\subsection{Simulation Setup}
In the simulation, the system consists of 100 MDs and 10 edge servers by default. Each MD has a randomly assigned computational capability $f_{n}$, with values uniformly distributed between $[10, 100]$ GFLOPS. We assume the computation capacity of each edge server $f_{m}$ follows uniform distribution of $[500, 2000]$ GFLOPS, and the storage capacity $C_{m}$ vary randomly within the range of $[2, 5]$ GB\cite{c33,c37}. We set $\alpha=1$. For the wireless communication setup, the transmission power of the server is fixed at 43 dBm, while each MD transmits at 23 dBm. The total available communication bandwidth for each server is set to 100 MHz for both uplink and downlink\cite{c43}. The channel noise is modeled with a Gaussian distribution having a standard deviation of 8 dB. The distance between the users and the base station is randomly distributed between 100 and 200 meters. The path loss exponent $\vartheta$ is set to 3.5, which is characteristic of a non-line-of-sight (NLOS) propagation environment. We assume the transmission data rate between the cloud and edge server is within the range of $[500, 700]$ Mbps\cite{c33,c37}.

Regarding the inference models, we consider nine DNN models that are commonly used in edge inference tasks, namely: VGG19, VGG16, VGG13, ResNet50, ResNet34, ResNet18, LeNet12, LeNet9, and LeNet7. These models represent various complexities and computational demands, ranging from relatively simple architectures like LeNet7 to more complex ones like ResNet50. For each of these models, we assume that multiple inference tasks are derived based on specific downstream requirements such as task precision or real-time processing constraints\cite{c43}. For example, each model can be assumed to give rise to 10 distinct inference services, resulting in a total of 90 inference services. Each MD's inference data size, which represents the amount of data to be processed for each service, is randomly distributed between 10 and 30 data units (e.g., number of images), simulating the variability of computational workloads typical of applications such as image classification, object recognition, and machine learning tasks in edge environments. Additionally, each MD has a data privacy leakage constraints, which is randomly distributed between 40\% and 70\% of the total inference raw data. The simulation results presented subsequently are averaged over 100 time slots. We assume the inference requests follows an uniform distribution.

We present the network architecture and parameters of VGG19 as an example, as shown in Table 1. This table provides detailed information on each layer’s data size (KB), computation load (MFLOPs), communication Load (KB), and privacy leakage loss (possibility), allowing readers to better understand the computational and communication characteristics of DNN models.

As the first work to jointly optimize model deployment, model partitioning, and user association under privacy constraints, we compare our framework with several representative baselines, including matching theory and DRL-based algorithms. These benchmarks, along with the widely adopted Full Local and Full Edge scenarios, comprehensively demonstrate the effectiveness and superiority of our proposed approach in balancing delay and privacy in edge intelligence systems:
\begin{itemize}
  \item \textbf{Full Local, FL}: In this scenario, all models are fully deployed and executed locally on each MD based on our proposed algorithms. This approach ensures complete data privacy, as no data is offloaded to the edge servers. The primary purpose of this algorithm is to provide a baseline for the inference delay when all computation remain entirely local, thus reflecting the worst-case latency scenario in terms of computation resource limitations on MDs.
  \item \textbf{Full Edge, FE}: The Full Edge algorithm assumes that all inference services are offloaded to the edge servers for computation. In this configuration, the edge servers handle the entire computational load, leading to full exposure of the data privacy. This setup serves as a benchmark to evaluate the inference delay when privacy is completely compromised, providing insights into the performance of edge-based computation without any privacy constraints. This setup serves as a benchmark to evaluate the inference delay when privacy is completely compromised, providing insights into the performance of edge-based inference without any privacy constraints.
  \item \textbf{Matching}: The Matching algorithm employs matching theory to solve the problem of allocating MDs to edge servers and selecting appropriate models for deployment\cite{c33}. This algorithm models the problem as a bipartite graph matching problem, where one set of nodes represents MDs and the other represents edge servers and models. The goal of the Matching algorithm is to minimize the average inference delay while adhering to the privacy constraints for single time slot.
  \item \textbf{LyDQN}: Based on the approach outlined in \cite{LyDQN}, the LyDQN algorithm utilizes Lyapunov optimization to transform the problem into a single-slot optimization problem. Then, the system seeks to minimize the long-term average inference delay while adhering to the privacy constraints by optimizing the joint server association, model deployment and model partition strategies in real time with Deep Q-Network (DQN) techniques.
\end{itemize}

\subsection{Simulation Results}
Fig. \ref{fig5} investigates the impact of the trade-off parameter $\alpha$ in (\ref{p2}) on the average delay and privacy loss. As shown in the figure, a larger $\alpha$ results in a faster decline and lower steady-state average delay, as the algorithm places greater emphasis on minimizing delay. In contrast, a smaller $\alpha$ shifts priority toward privacy, leading to a slight increase in average delay. Nevertheless, the privacy loss for all settings remains within acceptable bounds and gradually stabilizes over time slots. These results highlight that appropriate selection of $\alpha$ enables flexible control over the balance between latency and privacy, allowing the system to satisfy long-term privacy constraints while maintaining efficient performance for large-scale user scenarios. In the following simulation, we take the average results of $\alpha=1$.
\begin{figure}[t]
    \centering
    \subfigure[] {\includegraphics[width=1.72in,angle=0]{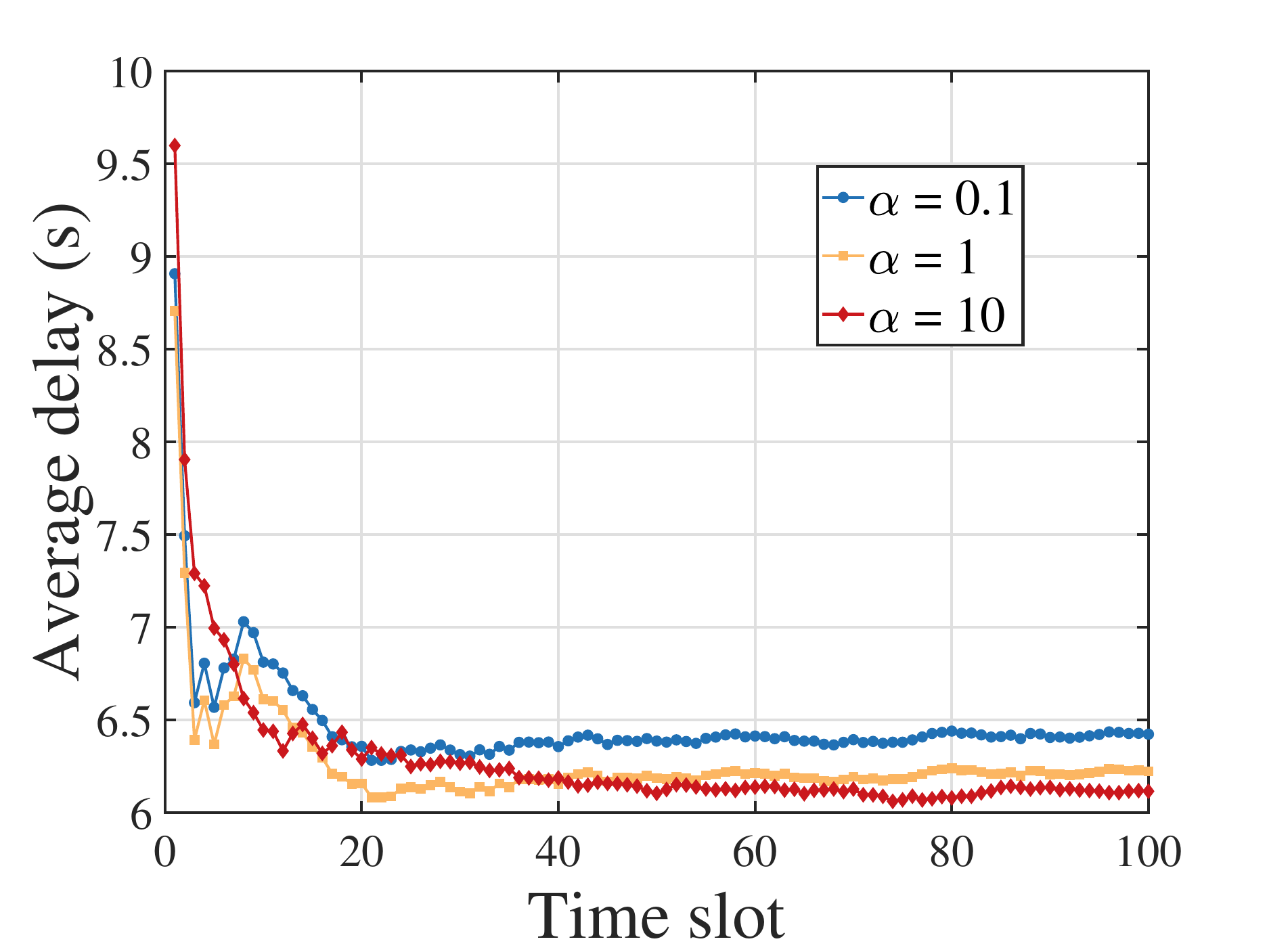}}
    \subfigure[] {\includegraphics[width=1.72in,angle=0]{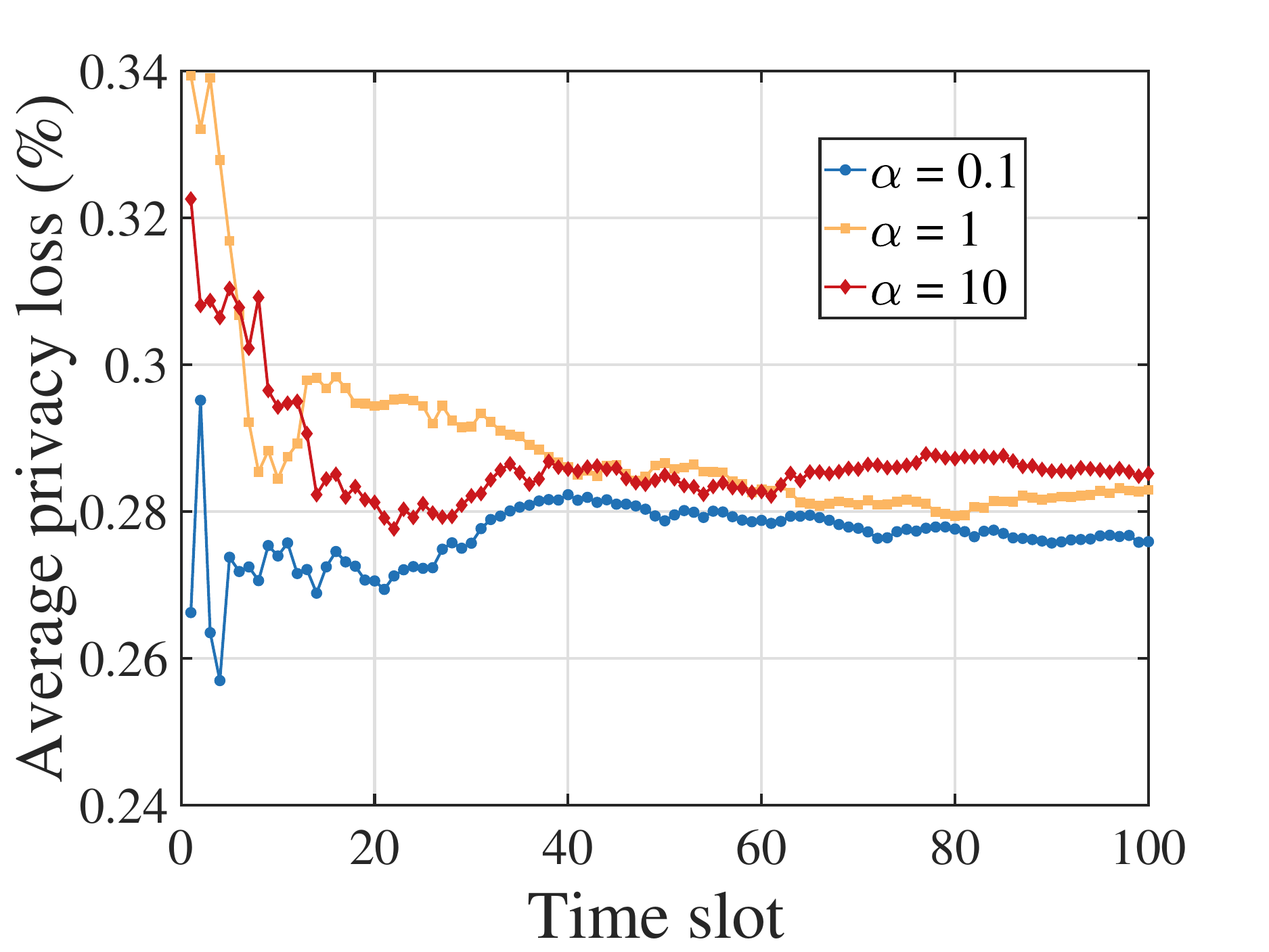}}
   \caption{Average Delay and Privacy Loss versus the Trade-off Parameter $\alpha$.}
    \label{fig5}
\end{figure}
\begin{figure}[t]
    \centering
    \subfigure[] {\includegraphics[width=1.72in,angle=0]{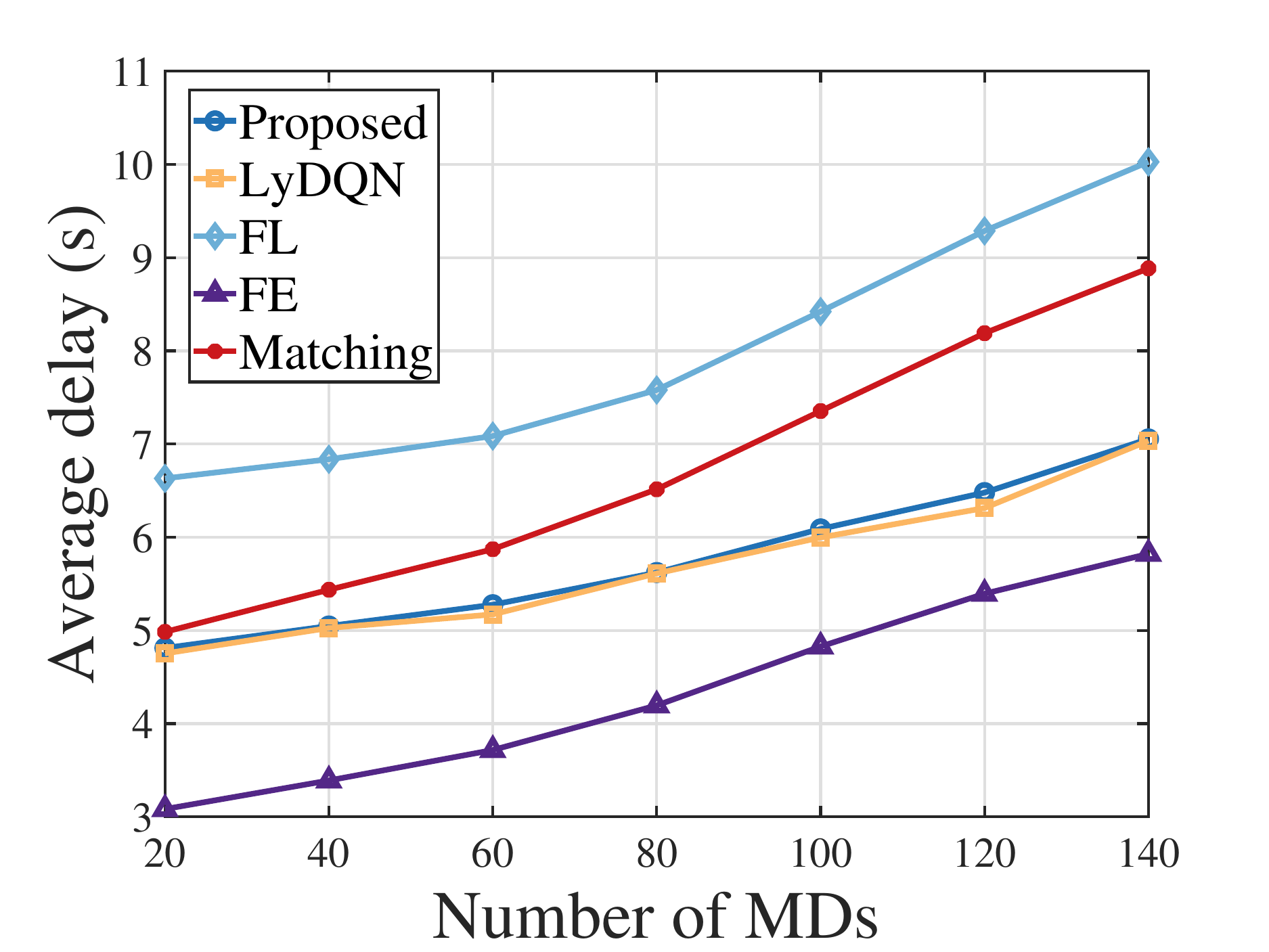}}
    \subfigure[] {\includegraphics[width=1.72in,angle=0]{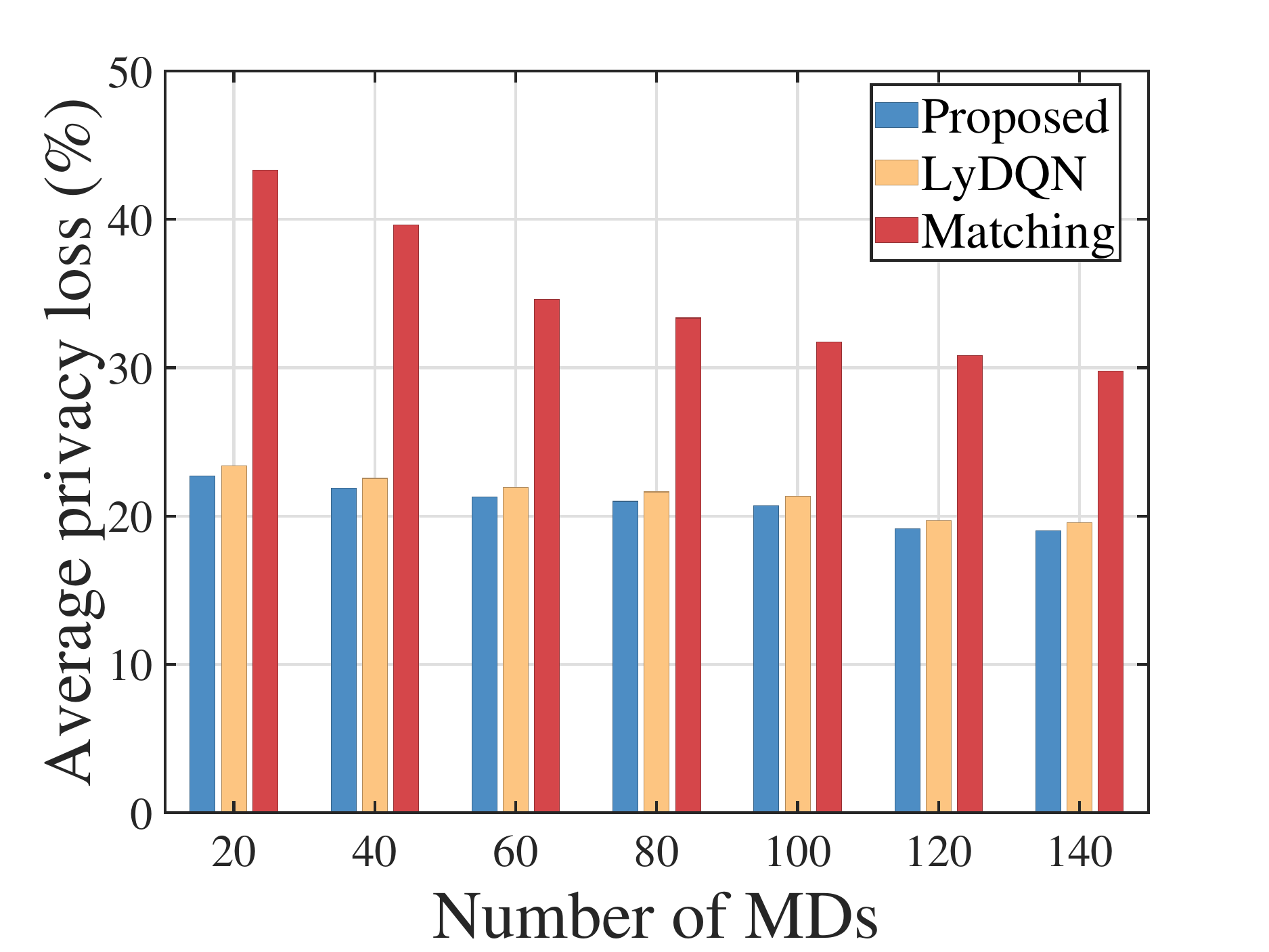}}
   \caption{Average delay and privacy loss versus the number of MDs.}
    \label{fig6}
\end{figure}
Fig. \ref{fig6} presents the average delay and privacy loss as the number of MDs increases, with the number of servers fixed at 10. As the number of MDs grows, the average delay increases across all algorithms. This is expected since more MDs compete for limited server resources and wireless bandwidth, leading to higher edge execution and transmission delays. In contrast, the average privacy loss decreases with more MDs, indicating that a larger proportion of tasks are executed locally, as offloading becomes less favorable under resource constraints. The proposed algorithm achieves a good balance between delay and privacy loss. While its delay is slightly higher than LyDQN, it is consistently lower than Matching and FL. In terms of privacy loss, both the proposed and LyDQN algorithms significantly outperform the Matching algorithms. The Matching algorithm exhibits higher privacy loss compared to both the proposed and LyDQN algorithms. This is because the Matching approach focuses primarily on minimizing inference delay by greedily assigning MDs to edge servers and models without explicitly considering long-term privacy constraints. As a result, more user data is likely offloaded to edge servers, increasing the risk of privacy leakage. In contrast, the proposed and LyDQN methods jointly optimize delay and privacy, ensuring that more computation is performed locally when needed to satisfy privacy requirements, thereby reducing overall privacy loss. Notably, although LyDQN yields comparable results to the proposed approach, its DQN model requires much longer training time and lacks the flexibility to quickly adapt to changing user and server scales. In contrast, the proposed method can adapt online, making it more practical for dynamic environments. FL and FE are not included in the privacy loss plot, since their privacy losses are always 0\% and 100\%, respectively.

\begin{figure}[t]
    \centering
    \subfigure[] {\includegraphics[width=1.72in,angle=0]{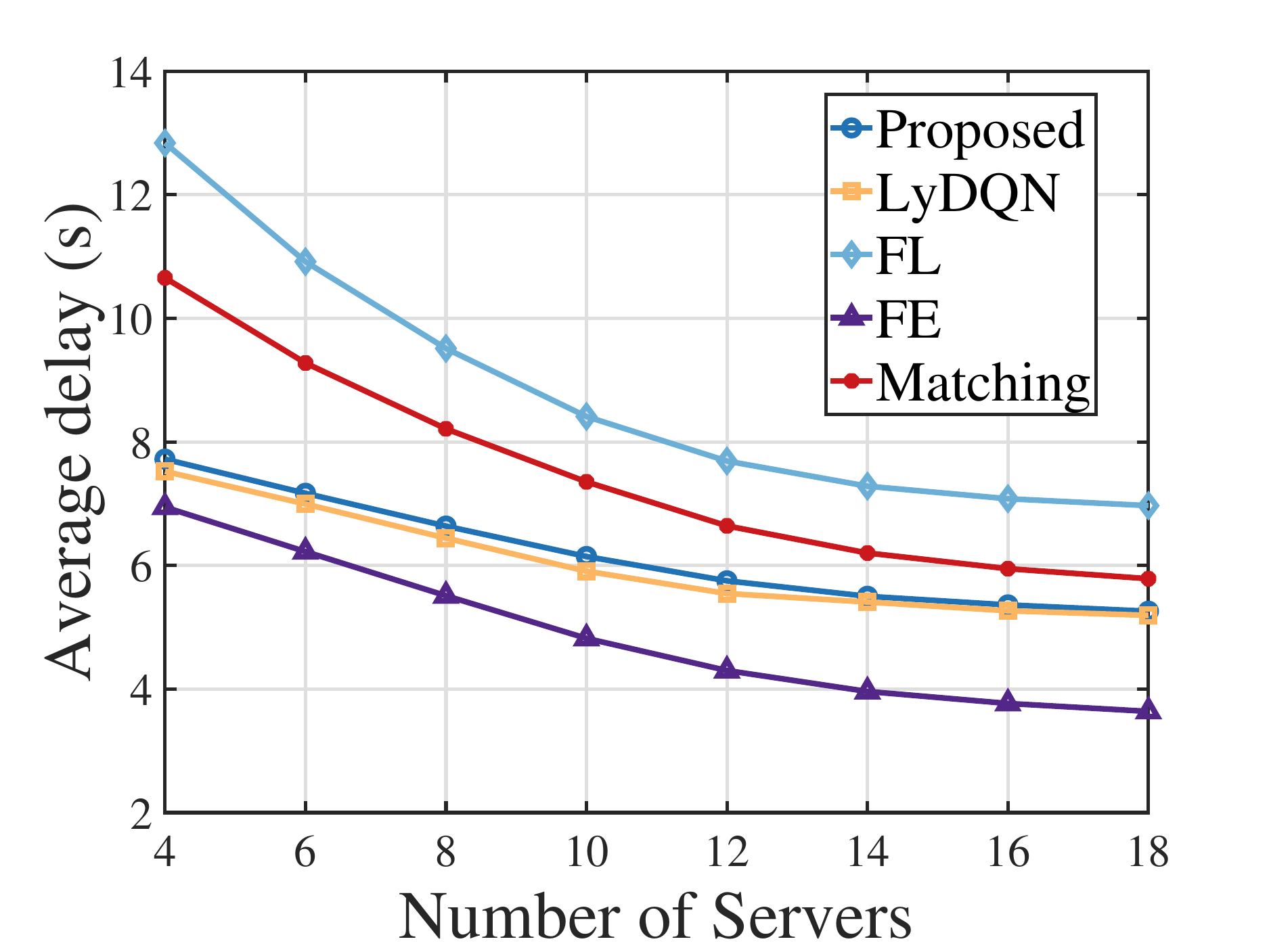}}
    \subfigure[] {\includegraphics[width=1.72in,angle=0]{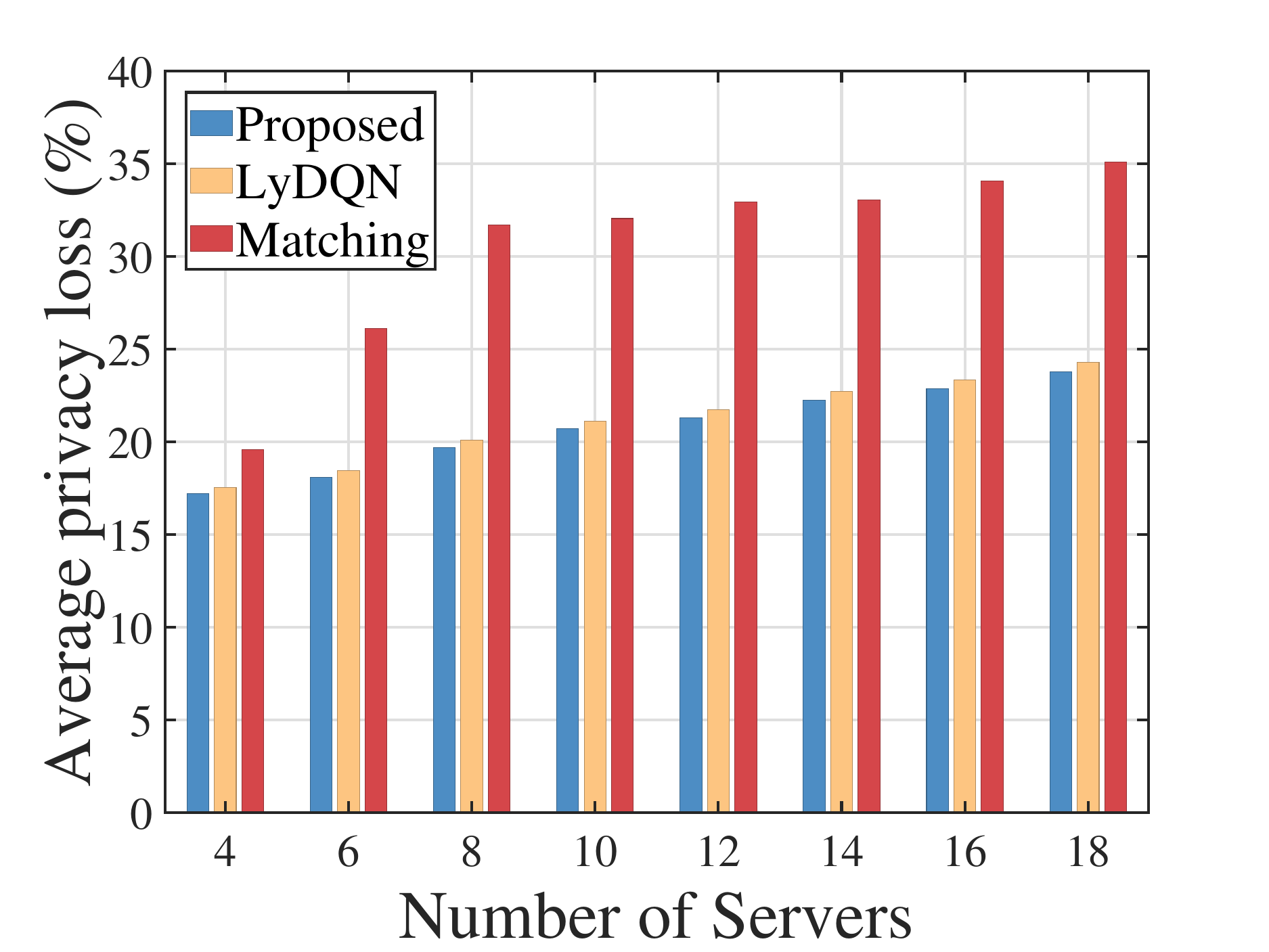}}
   \caption{Average delay and privacy loss versus the number of servers.}
    \label{fig7}
\end{figure}
\begin{figure}[t]
    \centering
    \subfigure[] {\includegraphics[width=1.72in,angle=0]{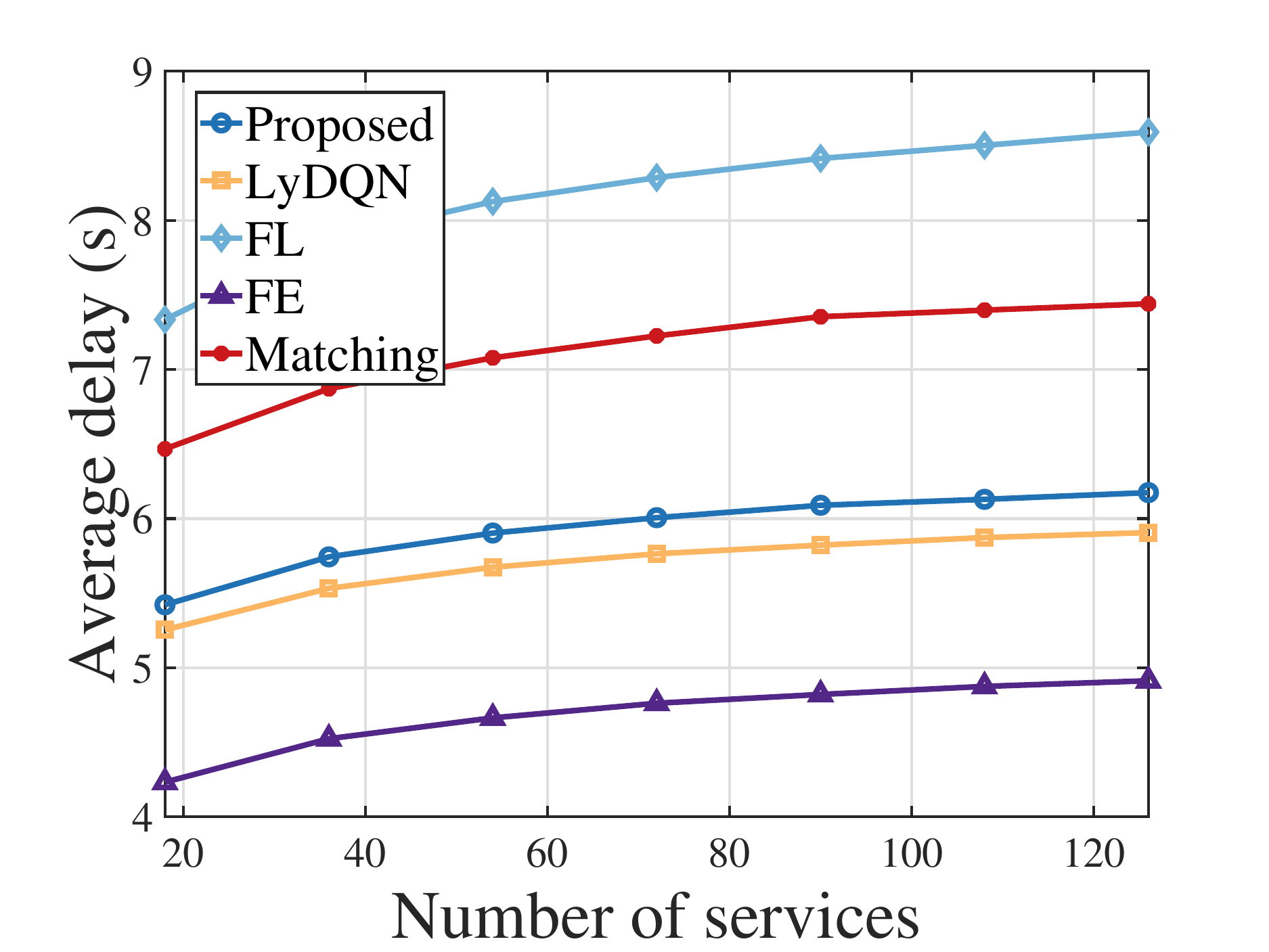}}
    \subfigure[] {\includegraphics[width=1.72in,angle=0]{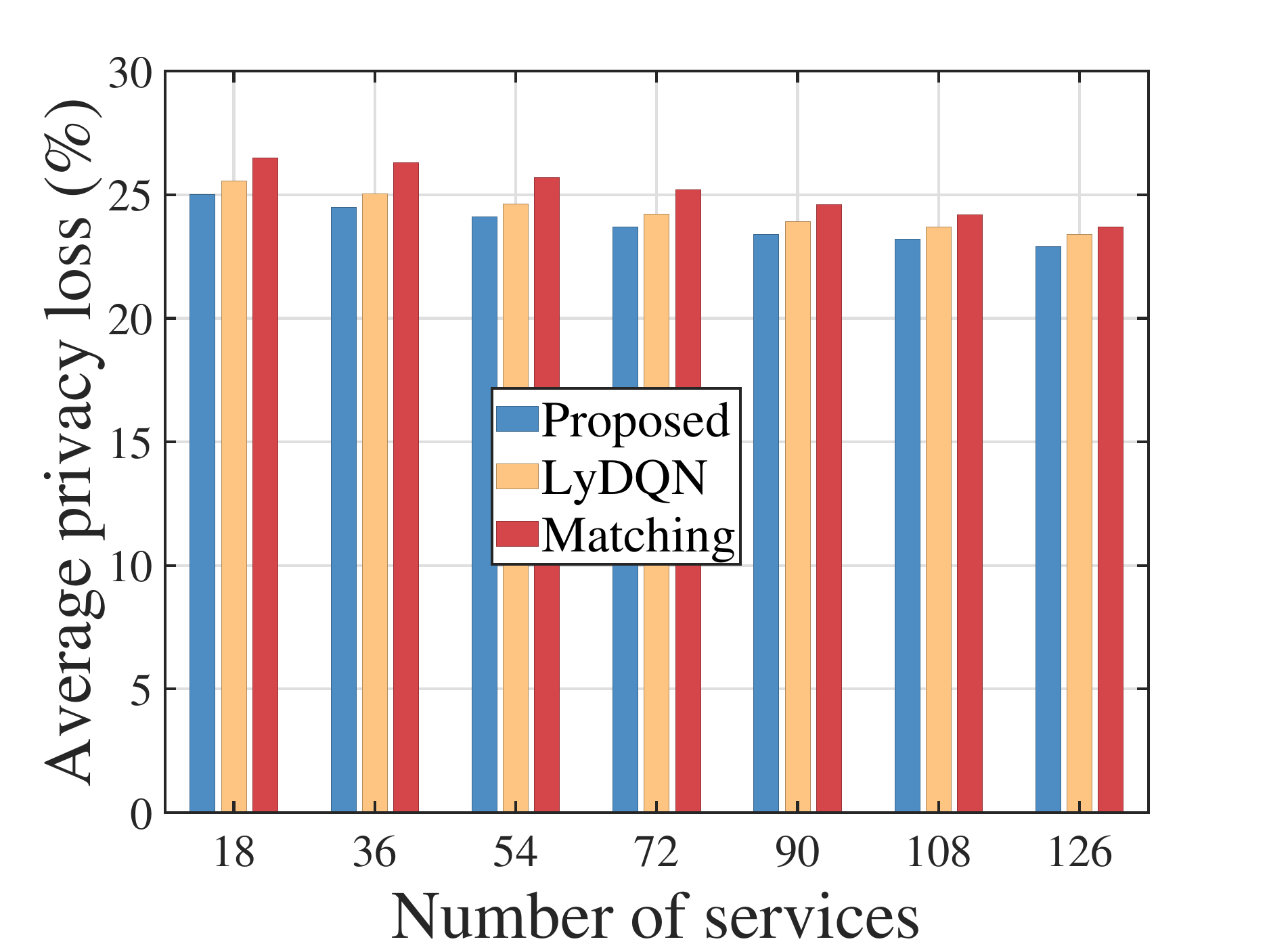}}
   \caption{Average delay and privacy loss versus the number of Inference services.}
    \label{fig8}
\end{figure}

Fig. \ref{fig7} shows that as the number of servers increases, the average delay decreases for all algorithms because more servers mean less resource contention among users, allowing more models to be offloaded to the edge and thus reducing delay. For example, the Proposed algorithm achieves about 28\% lower delay than FL and 18\% lower than Matching at 10 servers. However, this also leads to slightly higher privacy loss, as more model data is offloaded. Still, the Proposed algorithm keeps the privacy loss below 24\%, compared to over 34\% for Matching when there are 18 servers. Overall, the Proposed method achieves a better balance between delay and privacy as server resources increase.

Fig. \ref{fig8} shows that as the number of inference services increases, the average delay for all algorithms rises gradually, since more available services make it harder for users to be matched with their preferred models at edge servers. The Proposed and LyDQN algorithms consistently outperform Matching and FL in terms of lower delay. For example, with 126 services, the Proposed method reduces average delay by about 27\% compared to FL and by 15\% compared to Matching. On the other hand, the average privacy loss for all three algorithms slightly decreases as the number of services grows, as more MDs tend to process tasks locally. This result highlights the superior balance of the Proposed algorithm in managing both delay and privacy as service diversity increases.
\begin{figure}[t]
    \centering
    \subfigure[] {\includegraphics[width=1.72in,angle=0]{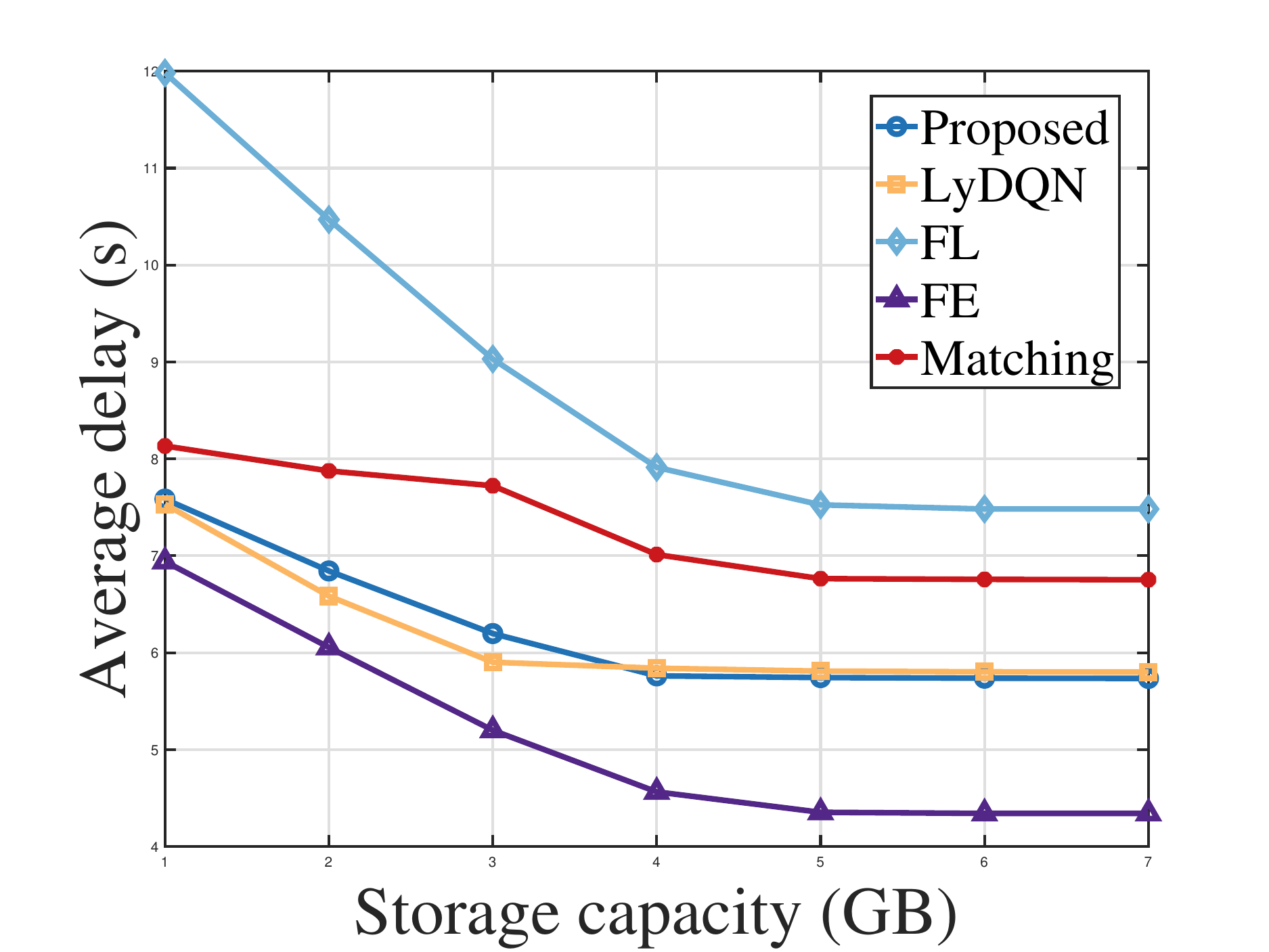}}
    \subfigure[] {\includegraphics[width=1.72in,angle=0]{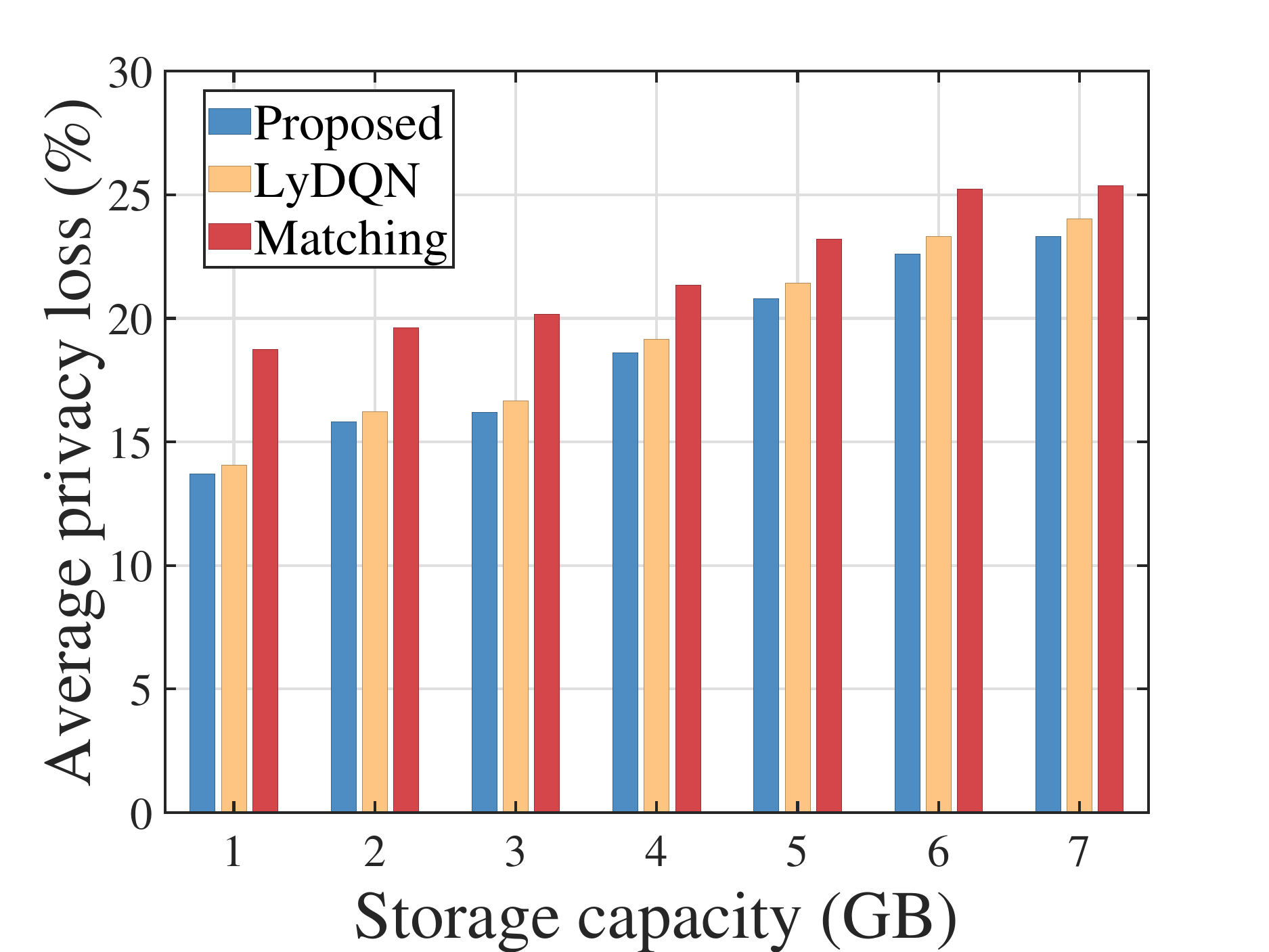}}
   \caption{Average delay and privacy loss versus the number of server storage capacity.}
    \label{fig9}
\end{figure}

Fig. \ref{fig9} demonstrates the impact of increasing storage capacity on both average delay and privacy loss. As storage capacity at each edge server increases, average delay significantly decreases for all algorithms, since more models can be cached and offloaded, transmission delays are reduced. For example, with a capacity of 7 GB, the Proposed algorithm reduces delay by approximately 40\% compared to FL, and by 22\% compared to Matching. However, the average privacy loss rises with increasing storage, as more models are processed at the edge instead of locally, leading to higher data exposure. Notably, the Proposed and LyDQN algorithms achieve a better trade-off than Matching, with lower delay and consistently less privacy loss for the same storage. This again shows the flexibility of the Proposed method in balancing delay and privacy as system resources scale.

Fig. \ref{fig10} presents the impact of varying privacy constraints on both average delay and privacy loss with 60 MDs and 10 servers. As the privacy constraint increases from 20\% to 90\%, the average delay of the proposed algorithm decreases modestly from approximately 4.5 to 3.8 s, representing a reduction of about 15\%. Meanwhile, the average privacy loss increases significantly from around 15\% to 33\%. This result demonstrates that, although relaxing the privacy constraint allows for more model partitioning and potential offloading, the improvement in delay is relatively limited. This is because the optimal split point selection is also influenced by server resources and network bandwidth, rather than privacy requirements alone.
\begin{figure}[t]
    \centering
    \subfigure[] {\includegraphics[width=1.72in,angle=0]{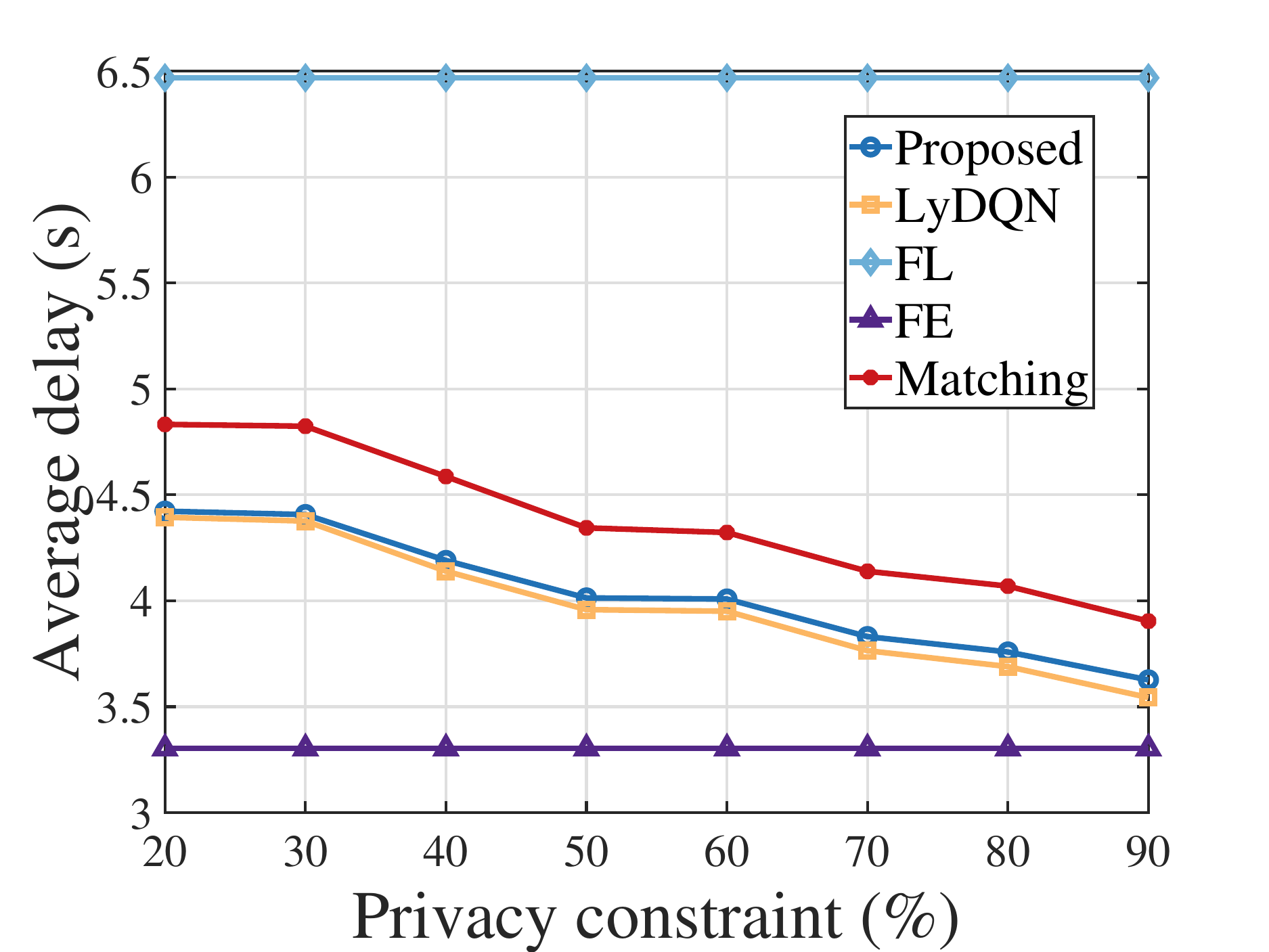}}
    \subfigure[] {\includegraphics[width=1.72in,angle=0]{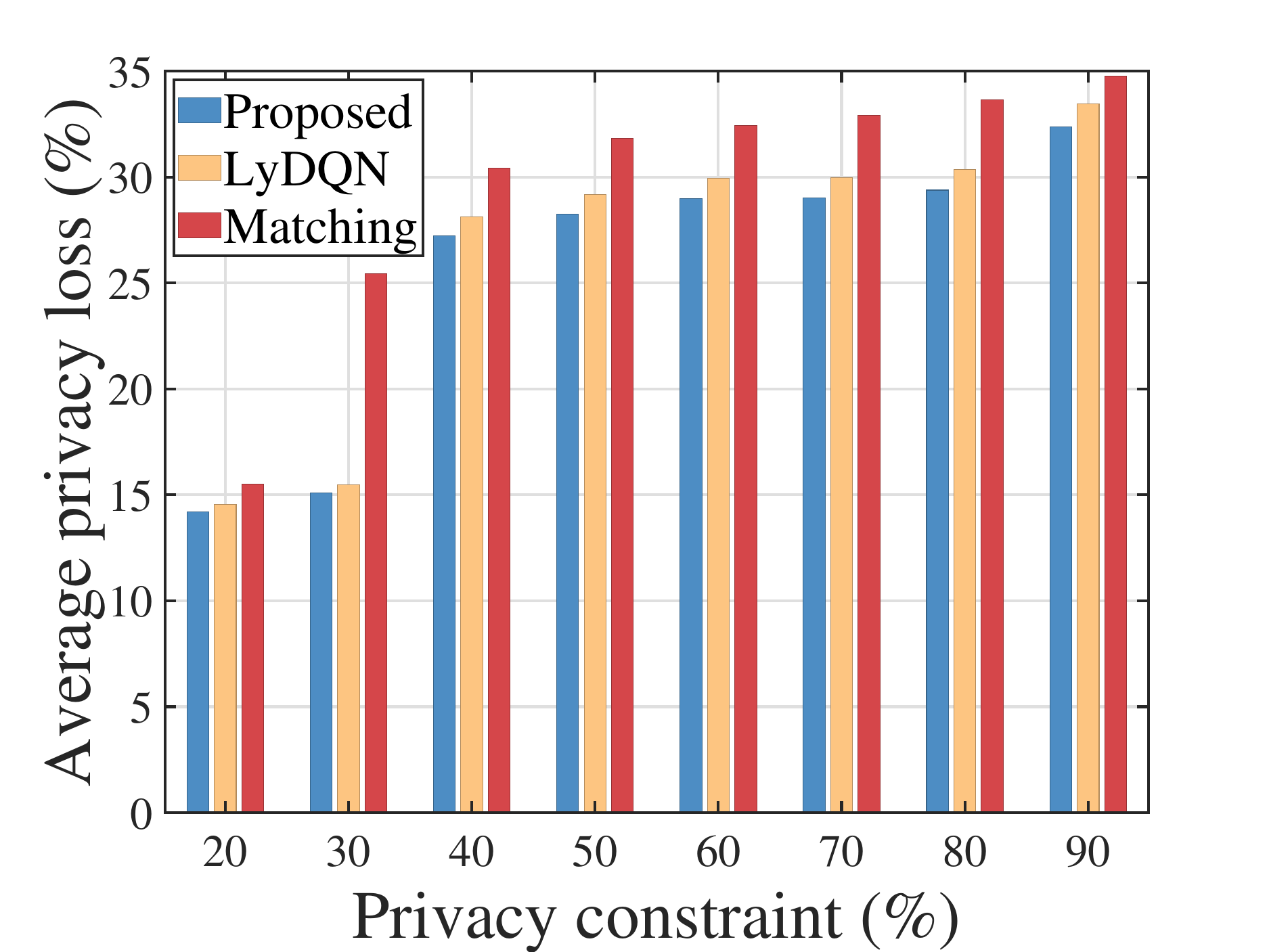}}
   \caption{Average delay and privacy loss versus the privacy constraint.}
    \label{fig10}
\end{figure}
\section{Conclusion}
In this paper, we addressed the critical challenges of delay-efficient and privacy-preserving collaborative EI by proposing a novel framework for joint DNN model deployment and partition optimization. Our approach leverages a Lyapunov-based optimization technique to handle dynamic and stochastic inference requests, ensuring long-term system stability and privacy constraints. By formulating the user-server association problem as a coalition formation game, we decoupled the complex joint optimization into manageable subproblems, enabling efficient suboptimal solutions through iterative optimization. The proposed greedy-based algorithm for model deployment and exhaustive search method for model partitioning further enhanced the system's performance. Extensive simulations demonstrated that our framework significantly reduces inference delay while maintaining stringent privacy constraints, outperforming existing baseline algorithms in various scenarios. The results highlight the importance of adaptive resource allocation, dynamic model partition, and privacy-aware optimization in edge computing environments. Future work could explore the integration of advanced machine learning techniques for real-time decision-making and further optimization of the trade-offs between computational efficiency, resource utilization, and data privacy in large-scale edge networks.


\begin{thebibliography}{1}
\addtolength{\itemsep}{-0.1em}
\bibitem{c1}
F. Xu, J. Xu, J. Chen, et al., "iGniter: Interference-Aware GPU Resource Provisioning for Predictable DNN Inference in the Cloud," \textit{IEEE Trans. Parallel Distrib. Syst.}, vol. 34, no. 3, pp. 812-827, Mar. 1, 2023.

\bibitem{c2}
A. Yousefpour, S. Devic, B. Q. Nguyen, et al., "Guardians of the Deep Fog: Failure-Resilient DNN Inference from Edge to Cloud," in \textit{Proc. First Int. Workshop Challenges Artif. Intell. Mach. Learn. Internet Things (AIChallengeIoT'19)}, New York, NY, USA, 2019, pp. 25-31.


\bibitem{c4}
M. M. H. Shuvo, S. K. Islam, J. Cheng, et al., "Efficient Acceleration of Deep Learning Inference on Resource-Constrained Edge Devices: A Review," \textit{Proc. IEEE}, vol. 111, no. 1, pp. 42-91, Jan. 2023.

\bibitem{c6}
N. Chen, J. Yang, Z. Cheng, et al., "GainNet: Coordinates the Odd Couple of Generative AI and 6G Networks," \textit{IEEE Network}, vol. 38, no. 5, pp. 56-65, Sept. 2024.

\bibitem{c7}
W. Qi, X. Guo, H. Du, "LMIE-BERT: A Learnable Method for Inter-Layer Ensembles to Accelerate Inference of BERT-Style Pre-Trained Models," in \textit{Proc. 9th Int. Conf. Big Data Comput. Commun. (BigCom)}, Hainan, China, 2023, pp. 271-277.

\bibitem{c8}
X. Xu, Y. Ding, S. X. Hu, et al., "Scaling for Edge Inference of Deep Neural Networks," \textit{Nat. Electron.}, vol. 1, pp. 216-222, 2018.

\bibitem{c9}
C. Jean, D. Brooks, K. Chen, et al., "Machine Learning at Facebook: Understanding Inference at the Edge," in \textit{Proc. IEEE Int. Symp. High Perform. Comput. Architecture (HPCA)}, Washington, DC, USA, Feb. 2019, pp. 331-344.

\bibitem{c10}
E. Li, L. Zeng, Z. Zhou, et al., "Edge AI: On-Demand Accelerating Deep Neural Network Inference via Edge Computing," \textit{IEEE Trans. Wireless Commun.}, vol. 19, no. 1, pp. 447-457, Jan. 2020.

\bibitem{c11}
J. Shao, J. Zhang, "Communication-Computation Trade-Off in Resource-Constrained Edge Inference," \textit{IEEE Commun. Mag.}, vol. 58, no. 12, pp. 20-26, Dec. 2020.

\bibitem{c12}
N. Shlezinger, I. V. Bajić, "Collaborative Inference for AI-Empowered IoT Devices," \textit{IEEE Internet Things Mag.}, vol. 5, no. 4, pp. 92-98, Dec. 2022.

\bibitem{c13}
X. Fan, Z. Cheng, M. Liwang, et al., "Multi-Fine-Grained DNNs Partition and Offloading over Fog Computing Networks," in \textit{Proc. IEEE Int. Conf. Commun. Workshops (ICC Workshops)}, Rome, Italy, Jun. 2023, pp. 599-604.

\bibitem{c14}
N. Ng, A. Souza, S. Diggavi, "Collaborative Inference in Resource-Constrained Edge Networks: Challenges and Opportunities," in \textit{Proc. MILCOM 2024 - IEEE Military Commun. Conf. (MILCOM)}, Washington, DC, USA, 2024, pp. 1-6.

\bibitem{c15}
W. Fan, Z. Chen, Z. Hao, et al., "DNN Deployment, Task Offloading, and Resource Allocation for Joint Task Inference in IIoT," \textit{IEEE Trans. Ind. Informat.}, vol. 19, no. 2, pp. 1634-1646, Feb. 2023.

\bibitem{c16}
Z. Liao, W. Hu, J. Huang, et al., "Joint Multi-User DNN Partitioning and Task Offloading in Mobile Edge Computing," \textit{Ad Hoc Networks}, vol. 144, May 2023, Art. no. 103156.

\bibitem{c17}
T. Mohammed, C. Joe-Wong, R. Babbar, "Distributed Inference Acceleration with Adaptive DNN Partitioning and Offloading," in \textit{Proc. IEEE INFOCOM 2020 - IEEE Conf. Comput. Commun.}, Toronto, ON, Canada, 2020, pp. 854-863.

\bibitem{c18}
Z. He, T. Zhang, R. B. Lee, "Model Inversion Attacks Against Collaborative Inference," in \textit{Proc. 35th Annu. Comput. Security Appl. Conf. (ACSAC '19)}, New York, NY, USA, 2019, pp. 148-162.

\bibitem{c19}
Z. He, T. Zhang, R. B. Lee, "Attacking and Protecting Data Privacy in Edge–Cloud Collaborative Inference Systems," \textit{IEEE Internet Things J.}, vol. 8, no. 12, pp. 9706-9716, Jun. 15, 2021.

\bibitem{c20}
R. Zhang, P. Isola, A. A. Efros, "The Unreasonable Effectiveness of Deep Features as a Perceptual Metric," in \textit{Proc. IEEE Conf. Comput. Vis. Pattern Recognit.}, 2018, pp. 586-595.

\bibitem{c21}
C. Hu, W. Bao, D. Wang, et al., "Dynamic Adaptive DNN Surgery for Inference Acceleration on the Edge," in \textit{Proc. IEEE INFOCOM 2019 - IEEE Conf. Comput. Commun.}, Paris, France, May 2019, pp. 1423-1431.

\bibitem{c22}
W. He, S. Guo, S. Guo, et al., "Joint DNN Partition Deployment and Resource Allocation for Delay-Sensitive Deep Learning Inference in IoT," \textit{IEEE Internet Things J.}, vol. 7, no. 10, pp. 9241-9254, Oct. 2020.

\bibitem{c23}
W. Q. Ren, Y. B. Qu, C. Dong, et al., "A Survey on Collaborative DNN Inference for Edge Intelligence," \textit{Mach. Intell. Res.}, vol. 20, pp. 370-395, Jun. 2023.

\bibitem{c24}
H. Hussain, P. S. Tamizharasan, C. S. Rahul, "Design Possibilities and Challenges of DNN Models: A Review on the Perspective of End Devices," \textit{Artif. Intell. Rev.}, vol. 55, pp. 5109-5167, Jun. 2022.

\bibitem{c25}
L. Cheng, Y. Gu, Q. Liu, et al., "Advancements in Accelerating Deep Neural Network Inference on AIoT Devices: A Survey," \textit{IEEE Trans. Sustain. Comput.}, vol. 9, no. 6, pp. 830-847, Nov.-Dec. 2024.

\bibitem{c26}
A. Eshratifar, M. S. Abrishami, M. Pedram, "JointDNN: An Efficient Training and Inference Engine for Intelligent Mobile Cloud Computing Services," \textit{IEEE Trans. Mobile Comput.}, vol. 20, no. 2, pp. 565-576, Feb. 2021.

\bibitem{c27}
L. Wang, X. Ren, C. Zhao, et al., "MPDM: A Multi-Paradigm Deployment Model for Large-Scale Edge-Cloud Intelligence," \textit{IEEE Internet Things J.}, vol. 10, no. 10, pp. 8773-8785, May 15, 2023.

\bibitem{c28}
M. Zyliński, A. Nassibi, I. Rakhmatulin, et al., "Deployment of Artificial Intelligence Models on Edge Devices: A Tutorial Brief," \textit{IEEE Trans. Circuits Syst. II: Express Briefs}, vol. 71, no. 3, pp. 1738-1748, Mar. 2024.

\bibitem{c29}
M. Xu, D. Niyato, H. Zhang, et al., "Joint Foundation Model Caching and Inference of Generative AI Services for Edge Intelligence," in \textit{Proc. IEEE Global Commun. Conf.}, 2023, pp. 3548-3553.

\bibitem{c30}
Y. Li, Z. Li, Z. Han, et al., "Automating Cloud Deployment for Real-Time Online Foundation Model Inference," \textit{IEEE/ACM Trans. Netw.}, vol. 32, no. 2, pp. 1509-1522, Apr. 2024.

\bibitem{c31}
P. Dai, B. Han, K. Li, et al., "Joint Optimization of Device Placement and Model Partitioning for Cooperative DNN Inference in Heterogeneous Edge Computing," \textit{IEEE Trans. Mobile Comput.}, vol. 24, no. 1, pp. 210-226, Jan. 2025.

\bibitem{c32}
J. Yan, S. Bi, Y.-J. A. Zhang, "Optimal Model Placement and Online Model Splitting for Device-Edge Co-Inference," \textit{IEEE Trans. Wireless Commun.}, vol. 21, no. 10, pp. 8354-8367, Oct. 2022.

\bibitem{c33}
Y. Wu, J. Wu, L. Chen, et al., "Share-Aware Joint Model Deployment and Task Offloading for Multi-Task Inference," \textit{IEEE Trans. Intell. Transp. Syst.}, vol. 25, no. 6, pp. 5674-5686, Jun. 2024.

\bibitem{c34}
J. Tang, G. Wu, M. M. Jalalzai, et al., "Energy-Optimal DNN Model Placement in UAV-Enabled Edge Computing Networks," \textit{Digit. Commun. Netw.}, 2023.

\bibitem{c35}
E. Baccour, A. Erbad, A. Mohamed, et al., "DistPrivacy: Privacy-Aware Distributed Deep Neural Networks in IoT Surveillance Systems," in \textit{Proc. IEEE Global Commun. Conf.}, 2020, pp. 1-6.

\bibitem{c36}
C. Shi, L. Chen, C. Shen, et al., "Privacy-Aware Edge Computing Based on Adaptive DNN Partitioning," in \textit{Proc. IEEE Int. Conf. Comput. Commun. (INFOCOM)}, 2019.

\bibitem{c37}
Y.-T. Yang, H.-Y. Wei, "A Coalition Formation Approach for Privacy and Energy-Aware Split Deep Learning Inference in Edge Camera Network," \textit{IEEE Trans. Netw. Service Manag.}, vol. 20, no. 3, pp. 3673-3685, Sep. 2023.

\bibitem{c38}
S. Fu, F. Dong, D. Shen, et al., "Privacy-Preserving Model Splitting and Quality-Aware Device Association for Federated Edge Learning," \textit{Software: Pract. Exp.}, vol. 54, no. 10, pp. 2063-2085, Oct. 2024.

\bibitem{c39}
G. Jiang, S. Han, X. Xu, et al., "Privacy-Aware Adaptive Model Splitting for Device-Edge Co-Inference," in \textit{Proc. IEEE Globecom Workshops (GC Wkshps)}, 2023, pp. 1-6.

\bibitem{c40}
E. Baccour, A. Erbad, A. Mohamed, et al., "RL-DistPrivacy: Privacy-Aware Distributed Deep Inference for Low Latency IoT Systems," \textit{IEEE Trans. Network Sci. Eng.}, vol. 9, no. 4, pp. 2066-2083, Jul.-Aug. 2022.

\bibitem{c41}
Y. Hu, X. Xu, L. Qi, et al., "Latency and Privacy Aware Convolutional Neural Network Distributed Inference for Reliable Artificial Intelligence Systems," \textit{IEEE Trans. Artif. Intell.}, early access, 2024.

\bibitem{c42}
M. Malekzadeh, F. Kawsar, "Salted Inference: Enhancing Privacy while Maintaining Efficiency of Split Inference in Mobile Computing," in \textit{Proc. 25th Int. Workshop Mobile Comput. Syst. Appl. (HOTMOBILE ’24)}, 2024, pp. 1-7.

\bibitem{c_RequestTSC}
T. Kim, C. K. Kim, S.-S. Lee, and S. Lee, "Incentive-Aware Partitioning and Offloading Scheme for Inference Services in Edge Computing," \textit{IEEE Trans. Serv. Comput.}, vol. 17, no. 4, pp. 1580--1592, Jul. 2024.

\bibitem{c_RequestInfocom}
H. Liu, S. Liu, S. Long, Q. Deng, and Z. Li, "Joint Optimization of Model Deployment for Freshness-Sensitive Task Assignment in Edge Intelligence," in \textit{Proc. IEEE Conf. Comput. Commun. (INFOCOM)}, Vancouver, BC, Canada, May 2024, pp. 1751--1760.

\bibitem{wang2004SSIM}
Z. Wang, A. C. Bovik, H. R. Sheikh, and E. P. Simoncelli, "Image Quality Assessment: From Error Visibility to Structural Similarity," \textit{IEEE Trans. Image Process.}, vol. 13, no. 4, pp. 600--612, Apr. 2004.

\bibitem{xiaTPDS}
X. Xia, F. Chen, Q. He, et al., "Formulating Cost-Effective Data Distribution Strategies Online for Edge Cache Systems," \textit{IEEE Trans. Parallel Distrib. Syst.}, vol. 33, no. 12, pp. 4270-4281, 2022.

\bibitem{c43}
G. Qu, Z. Lin, F. Liu, et al., "TrimCaching: Parameter-Sharing AI Model Caching in Wireless Edge Networks," in \textit{Proc. 44th Int. Conf. Distrib. Comput. Syst. (ICDCS)}, 2024.

\bibitem{LyDQN}
Y. Liang, H. Tang, H. Wu, et al., "Lyapunov-Guided Offloading Optimization Based on Soft Actor-Critic for ISAC-Aided Internet of Vehicles," \textit{IEEE Trans. Mobile Comput.}, vol. 23, no. 12, pp. 14708-14721, Dec. 2024.

\end{thebibliography}
\end{document}